\documentclass[journal]{IEEEtran}
\usepackage{booktabs,makecell}

\usepackage[group-separator={,},group-minimum-digits=4]{siunitx}
\usepackage{graphicx}
\usepackage{color,xcolor}
\usepackage{amsthm}
\usepackage{amsfonts}
\usepackage{mathrsfs}
\usepackage{amsmath}
\usepackage{caption}
\usepackage{amssymb}
\usepackage{algorithm}               
\usepackage{algorithmic}             
\usepackage{url}
\newtheorem{problem}{Problem}
\newtheorem{definition}{Definition}
\newtheorem{assumption}{Assumption}
\newtheorem{lemma}{Lemma}
\newtheorem{theorem}{Theorem}

\newcommand{\tabincell}[2]{\begin{tabular}{@{}#1@{}}#2\end{tabular}}
\newenvironment{sproof}{%
  \proof}{\endproof}
\begin{document}
\title{Towards Quantized Model Parallelism for Graph-Augmented MLPs Based on Gradient-Free ADMM Framework}

\author{Junxiang Wang\thanks{* Junxiang Wang and Hongyi Li contribute equally to this work, and Yongchao Wang and Liang Zhao are corresponding authors.}\IEEEauthorrefmark{2}, Hongyi Li\IEEEauthorrefmark{3}, Zheng Chai\IEEEauthorrefmark{4}, Yongchao Wang\IEEEauthorrefmark{3}, Yue Cheng\IEEEauthorrefmark{4} and Liang Zhao\IEEEauthorrefmark{2}\\
\IEEEauthorblockA{\IEEEauthorrefmark{2}Department of Computer Science and Informatics,
Emory University, Atlanta, Georgia, USA, 30030}\\
\IEEEauthorblockA{\IEEEauthorrefmark{3}The State Key Laboratory of Integrated Service Networks, Xidian University, Xi'an, Shaanxi, China, 710071}\\
\IEEEauthorblockA{\IEEEauthorrefmark{4} Department of Computer Science, University of Virginia, Charlottesville, Virginia, USA, 22904}
}
\maketitle
\begin{abstract}
\indent While Graph Neural Networks (GNNs) are popular in the deep learning community, they suffer from several challenges including over-smoothing, over-squashing, and gradient vanishing. Recently, a series of models have attempted to relieve these issues by first augmenting the node features and then imposing node-wise functions based on Multi-Layer Perceptron (MLP), which are widely referred to as GA-MLP models. However, while GA-MLP models enjoy deeper architectures for better accuracy, their efficiency largely deteriorates. Moreover, popular acceleration techniques such as stochastic-version or data-parallelism cannot be effectively applied due to the dependency among samples (i.e., nodes) in graphs. To address these issues, in this paper, instead of data parallelism, we propose a  parallel graph deep learning Alternating Direction Method of Multipliers (pdADMM-G) framework to achieve model parallelism: parameters in each layer of GA-MLP models can be updated in parallel. The extended pdADMM-G-Q algorithm reduces communication costs by introducing the quantization technique. Theoretical convergence to a (quantized) stationary point of the pdADMM-G algorithm and the pdADMM-G-Q algorithm is provided with a sublinear convergence rate $o(1/k)$, where $k$ is the number of iterations. Extensive experiments demonstrate the convergence of two proposed algorithms. Moreover, they lead to a more massive speedup and better performance than all state-of-the-art comparison methods on nine benchmark datasets. Last but not least, the proposed pdADMM-G-Q algorithm reduces communication overheads by up to $45\%$ without loss of performance. Our code is available at \url{https://github.com/xianggebenben/pdADMM-G}.
\end{abstract}
\section{Introduction}
\indent Graph Neural Networks (GNNs) have accomplished state-of-the-art performance in various graph applications such as node classification and link prediction.  This is because they handle graph-structured data via aggregating neighbor information and extending operations and definitions of the deep learning approach \cite{kipf2017semi}. However, their performance has significantly been restricted via their depths due to the over-smoothing problem (i.e. the representations of different nodes in a graph tend to be similar when stacking multiple layers) \cite{chen2021graph}, the over-squashing problem (i.e. the information flow among distant nodes distorts along the long-distance interactions) \cite{topping2022understanding}, and the gradient vanishing problem (i.e. the signals of gradients decay with the depths of GNN models) \cite{chen2021graph}. These challenges still exist even though some models such as GraphSAGE \cite{hamilton2017inductive} have been proposed to alleviate them.\\
\indent On the other hand, the Graph Augmented Multi-Layer Perceptron (GA-MLP) models have recently received fast-increasing attention as an alternative to deal with the aforementioned drawbacks of conventional GNNs via the augmentation of graph features. GA-MLP models augment node representations of graphs and feed them into  Multi-Layer Perceptron (MLP) models. Compared with GNNs,  GA-MLP models are more resistant to the over-smoothing problem \cite{chen2021graph} and therefore demonstrate outstanding performance. For example, Wu et al. showed that a two-layer GA-MLP approximates the performance of the GNN models on multiple datasets \cite{wu2019simplifying}.\\
\indent  GA-MLP models are supposed to perform better with the increase of their depths. However, similar to GNNs, GA-MLP models still suffer from the gradient vanishing problem, which is caused by the mechanism of the classic backpropagation algorithm. This is because gradient signals diminish during the transmission among deep layers.  Moreover, while the models go deeper, efficiency will become an issue, especially for medium- and large-size graphs. Compared to the data such as images and texts, where identically and independently distributed (i.i.d.) samples are assumed, efficiency issues in graph data are much more difficult to handle due to the dependency among data samples (i.e., nodes in graphs). Such dependency seriously troubles the effectiveness of using typical acceleration techniques such as sampling-based methods, and data-parallelism distributed learning in solving the efficiency issue. Therefore, parallelizing the computation along layers is a natural workaround, but the backpropagation prevents the gradients of different layers from being calculated in parallel. This is because the calculation of the gradient in one layer is dependent on its previous layers. \\
\indent To handle these challenges, recently gradient-free optimization methods such as the Alternating Direction Method of Multipliers (ADMM) have been investigated to overcome the difficulties of the backpropagation algorithm. The spirit of ADMM is to decouple a neural network into layerwise subproblems such that each of them can be solved efficiently. ADMM does not require gradient calculation and therefore can avoid the gradient vanishing problem. Existing literature has shown its great potential.  For example, Talyor et al. and Wang et al. proposed ADMM to train MLP models \cite{taylor2016training,wang2019admm}. Extensive experiments have demonstrated that the ADMM has outperformed most comparison methods such as Gradient Descent (GD).\\
\indent   In this paper, we propose a novel parallel graph deep learning Alternating Direction Method of Multipliers (pdADMM-G) optimization framework to train large-scale GA-MLP models, and the extended pdADMM-G-Q algorithm reduces the communication cost of the pdADMM-G algorithm by the quantization techniques. Our contributions to this paper include:
\begin{itemize}
\item We propose a novel reformulation of  GA-MLP models, which splits a neural network into independent layer partitions and allow for ADMM to achieve model parallelism.
\item We propose a novel pdADMM-G framework to train a GA-MLP model. All subproblems generated by the ADMM algorithm are discussed. The extended pdADMM-G-Q algorithm reduces communication costs by introducing the quantization technique.
\item  We provide the theoretical convergence guarantee of the proposed pdADMM-G algorithm and the pdADMM-G-Q algorithm. Specifically, they converge to a (quantized) stationary point of GA-MLP models when the hyperparameters are sufficiently large, and their sublinear convergence rates are $o(1/k)$.
\item We conduct extensive experiments on nine benchmark datasets to show the convergence, the massive speedup of the proposed pdADMM-G algorithm and the pdADMM-G-Q algorithm, as well as their outstanding performance when compared with all state-of-the-art optimizers. Moreover, the proposed pdADMM-G-Q algorithm reduces communication overheads by up to $45\%$. 
\end{itemize}
\indent The organization of this paper is shown as follows: In Section \ref{sec:related work}, we summarize recent related research work to this paper. In Section \ref{sec:methods}, we propose the pdADMM-G algorithm and the pdADMM-G-Q algorithm to train deep GA-MLP models. Section \ref{sec:convergence} details the convergence properties of the proposed pdADMM-G algorithm and the pdADMM-G-Q algorithm.
Extensive experiments on nine benchmark datasets to demonstrate the convergence, speedup, communication savings, and outstanding performance of the pdADMM-G algorithm and the pdADMM-G-Q algorithm are shown in Section \ref{sec:experiments}, and Section \ref{sec:conclusion} concludes this work.
\section{Related Work}
\label{sec:related work}
\indent This section summarizes existing literature related to this research.\\
\indent \textbf{Distributed Deep Learning.} With the increase of public datasets and layers of neural networks, it is imperative to establish distributed deep learning systems for large-scale applications. Many systems have been established to satisfy such needs. Famous systems include  Terngrad  \cite{wen2017terngrad}, Horovod \cite{sergeev2018horovod}, SINGA \cite{ooi2015singa} Mxnet \cite{chen2015mxnet}, TicTac \cite{hashemitictac} and Poseidon \cite{zhang2017poseidon}. They applied some parallelism techniques to reduce computational time, and therefore improved the speedup. Existing parallelism techniques can be classified into two categories: data parallelism and model parallelism.
Data parallelism focuses on distributing data across different processors and then aggregating results from them into a server. Scaling GD is one of the most common ways to reach data parallelism \cite{zinkevich2010parallelized}. For example, the distributed architecture, Poseidon, is achieved by scaling GD through overlapping communication and computation over networks.  The recently proposed ADMM \cite{taylor2016training,wang2019admm} is another way to achieve data parallelism. However, data parallelism suffers from the bottleneck of a neural network: for GD, the gradient should be transmitted through all processors;  for ADMM, the parameters in one layer are subject to those in its previous layer. As a result, this leads to heavy communication costs and time delays. Model parallelism, however, can solve this challenge because model parallelism splits a neural network into many independent partitions. In this way, each partition can be optimized independently and reduce layer dependency. For instance, Parpas and Muir proposed a parallel-in-time method from the perspective of dynamic systems \cite{parpas2019predict}; Huo et al. introduced a feature replay algorithm to achieve model parallelism \cite{huo2018training}. Zhuang et al. broke layer dependency by introducing the delayed gradient \cite{zhuang2019fully}.\\
\indent \textbf{Deep Learning on Graphs.} Graphs are ubiquitous structures and are popular in real-world applications. There is a surge of interest to apply deep learning techniques to graphs. For a comprehensive summary please refer to \cite{wu2020comprehensive}. It classified existing GNN models into four categories:  Recurrent Graph Neural Networks (RecGNNs), Convolutional Graph Neural Networks (ConvGNNs), Graph Autoencoders (GAEs), and Spatial-Temporal Graph Neural Networks (STGNNs). RecGNNs learn node representation with recurrent neural networks via the message passing mechanisms \cite{gallicchio2010graph,li2015gated,dai2018learning}; ConvGNNs generalize the operations of convolution to graph data and stack multiple convolution layers to extract high-level node features \cite{bruna2013spectral,henaff2015deep,defferrard2016convolutional}; GAEs encode node information into a latent space and reconstruct graphs from the encoded node representation \cite{cao2016deep,wang2016structural,pan2018adversarially}; the idea of STGNNs is to capture spatial dependency and temporal dependency simultaneously \cite{seo2018structured,li2018diffusion,jain2016structural}.
\section{The pdADMM-G Algorithm}
\label{sec:methods}
\indent We propose the pdADMM-G algorithm to solve GA-MLP models in this section. Specifically, Section \ref{sec:problem formulation} formulates the GA-MLP model training problem , and Section \ref{sec:subproblem} proposes the pdADMM-G algorithm. Section \ref{sec:pdADMM-G-q} extends the proposed pdADMM-G algorithm to the pdADMM-G-Q algorithm for quantization.
\subsection{Problem Formulation}
\label{sec:problem formulation}
\begin{table}
\centering
 \begin{tabular}{cc}
 \hline
 Notations&Descriptions\\ \hline
 $L$& Number of layers.\\
 $W_l$& The weight matrix for the $l$-th layer.\\
 $b_l$& The intercept vector for the $l$-th layer.\\
 $z_l$& The auxiliary variable of the linear mapping for the $l$-th layer.\\
 $f_l(z_l)$& The nonlinear activation function for the $l$-th layer.\\
 $p_l$& The input for the $l$-th layer.\\
 $q_l$& The output for the $l$-th layer.\\
 $X$& The node representation of the graph.\\
 $A$& The adjacency matrix of the graph.\\
 $y$& The predefined label vector.\\
 $R(z_L,y)$& The risk function for the $L$-th layer.\\
 $n_l$& The number of neurons for the $l$-th layer.\\
 $u_l$& The dual variable for the $l$-th layer.\\
\hline
  \end{tabular}
  \captionof{table}{Important Notations}
   \label{tab:notation}

\end{table}
Consider a graph $G=(V, E)$, where $V$ and $E$ are sets of nodes and edges, respectively, $\vert V\vert$ is the number of nodes,  let $\Psi=\{\psi_1(A),\cdots,\psi_K(A)\}$ be a set of (usually multi-hop) operators $\psi_i(A):\mathbb{R}^{\vert V\vert}\rightarrow \mathbb{R}^{\vert V\vert}(i=1,\cdots,K)$ that are functions of the adjacency matrix $A\in \{0,1\}^{\vert V\vert\times\vert V\vert}$, and  $\mathbb{R}^{\vert V\vert}$ is the domain of $\psi_i(A) \ (i=1,\cdots,K)$. ${X}_k=H\psi_k(A)$ is the augmentation of node features by the $k$-hop operator, where $H\in \mathbb{R}^{d\times\vert V\vert}$ is a matrix of node features, 
and $d$ is the dimension of features.  $X_k(k=1,\cdots,K)$ are stacked into $X=[{X}_1;\cdots;{X}_K]$ by column. Then the GA-MLP training problem is formulated as follows \cite{wang2019admm}:
\begin{problem}
\label{prob:problem 1}
\begin{align*}
     & \min\nolimits_{W_l,b_l,z_l,p_l} R(z_L;y), \\
     &s.t.\ z_l=W_lp_{l}+b_l, \   p_{l+1}=f_l(z_l)(l=1,\cdots,L-1),
\end{align*}
\end{problem}
where $p_1={X}\in \mathbb{R}^{{n_0}\times \vert V\vert}$ is the input of deep GA-MLP models, where $n_0=Kd$ is the dimension of input and $y$ is a predefined label vector. $p_l\in\mathbb{R}^{n_l\times \vert V\vert}$ is the input for the $l$-th layer, also the output for the $(l-1)$-th layer, and $n_l$ is the number of neurons for the $l$-th layer. $R(z_L;y)$ is a risk function for the $L$-th layer, which is convex and continuous; $z_l=W_lp_l+b_l$  and $p_{l+1}=f_l(z_l)$ are linear and nonlinear mappings for the $l$-th layer, respectively, and $W_l\in\mathbb{R}^{{n_l}\times{n_{l-1}}}$ and $b_l\in\mathbb{R}^{n_l}$ are the weight matrix and the intercept vector for the $l$-th layer, respectively.\\
\indent In Problem \ref{prob:problem 1}, $\Psi$ can be considered as a prepossessing step to augment node features via $A$, and hence it is predefined. One common choice can be $\Psi=\{I,A,A^2,\cdots,A^{K-1}\}$.\\
\indent Problem \ref{prob:problem 1} can be addressed by deep learning Alternating Direction Method of Multipliers (dlADMM) \cite{wang2019admm}. However, parameters in one layer are dependent on its neighboring layers and hence can not achieve parallelism. For example, the update of $p_{l+1}$ on the $(l+1)$-th layer needs to wait before $z_l$ on the $l$-th layer is updated.
\label{sec:pdADMM-G}
In order to address layer dependency, we relax Problem \ref{prob:problem 1} to Problem \ref{prob:problem 2} as follows:
\begin{problem}
\label{prob:problem 2}
\begin{align*}
    &\min\nolimits_{\textbf{p},\textbf{W},\textbf{b},\textbf{z},\textbf{q}} F(\textbf{p},\textbf{W},\textbf{b},\textbf{z},\textbf{q})=R(z_L;y)\\&+(\nu/2)(\sum\nolimits_{l=1}^{L}\Vert z_l-W_lp_l-b_l\Vert^2_2+\sum\nolimits_{l=1}^{L-1}\Vert q_l-f_l(z_l)\Vert^2_2),\\
    &s.t. \ p_{l+1}=q_l,
\end{align*}
\end{problem}
where   $\textbf{p}=\{p_l\}_{l=1}^{L}$, $\textbf{W}=\{W_l\}_{l=1}^{L}$, $\textbf{b}=\{b_l\}_{l=1}^{L}$, $\textbf{z}=\{z_l\}_{l=1}^{L}$, $\textbf{q}=\{q_l\}_{l=1}^{L-1}$, and $\nu>0$ is a tuning parameter. As $\nu\rightarrow \infty$, Problem \ref{prob:problem 2} approaches Problem  \ref{prob:problem 1}. We reduce layer dependency by splitting the output of the $l$-th layer and the input of the ($l+1$)-th layer into two variables $p_{l+1}$ and $q_{l}$, respectively.
\subsection{The pdADMM-G Algorithm}
\label{sec:subproblem}
\indent  The high-level overview of the pdADMM-G algorithm is shown in Figure \ref{fig:pdADMM-G framework}. Specifically, the inputs of GA-MLP models are augmented by $H\psi_k(A) \ (k=1,\cdots,K)$, and then GA-MLP models are split into multiple layers, each of which can be optimized by an independent client. Therefore, layerwise training can be implemented in parallel. Moreover, the gradient vanishing problem can be avoided in this way. This is because the accumulated gradient calculated by the backpropagation algorithm is split into layerwise components. \\
\indent  Now we follow the ADMM routine to solve Problem \ref{prob:problem 2}. The augmented Lagrangian function is formulated mathematically as follows:
\begin{align*}
    &L_\rho(\textbf{p},\textbf{W},\textbf{b},\textbf{z},\textbf{q},\textbf{u})\\&\!=\!F(\textbf{p},\textbf{W},\textbf{b},\textbf{z},\textbf{q})\!+\!\sum\nolimits_{l\!=\!1}^{L\!-\!1}(u_l^T(p_{l\!+\!1}\!-\!q_l)\!+\!(\rho/2)\Vert p_{l+1}-q_l\Vert^2_2)\\&=\!R(z_L;y)\!+\!\phi(p_1,W_1,b_1,z_1)\!+\!\sum\nolimits_{l\!=\!2}^{L}\!\phi(p_l,W_l,b_l,z_l,q_{l-1},u_{l\!-\!1})\!\\&+\!(\nu/2)\sum\nolimits_{l\!=\!1}^{L-1}\Vert q_l\!-\!f_l(z_l)\Vert^2_2,
\end{align*}
where $\phi(p_1,W_1,b_1,z_1)=(\nu/2)\Vert z_1-W_1p_1-b_1\Vert^2_2$, $\phi(p_l,W_l,b_l,z_l,q_{l-1},u_{l-1})=(\nu/2)\Vert z_l-W_lp_l-b_l\Vert^2_2+u^T_{l-1}(p_l-q_{l-1})+(\rho/2)\Vert p_l-q_{l-1}\Vert^2_2$,  $u_l(l=1,\cdots,L-1)$ are dual variables, $\rho>0$ is a parameter, and $\textbf{u}=\{u_l\}_{l=1}^{L-1}$.
The detail of the pdADMM-G algorithm is shown in Algorithm \ref{algo:distributed ADMM}. Specifically, Lines 5-9 update primal variables $\textbf{p}$, $\textbf{W}$, $\textbf{b}$, $\textbf{z}$ and $\textbf{q}$, respectively, while Line 11 updates the dual variable $\textbf{u}$.
\begin{figure}
   \centering
    \includegraphics[width=\linewidth]{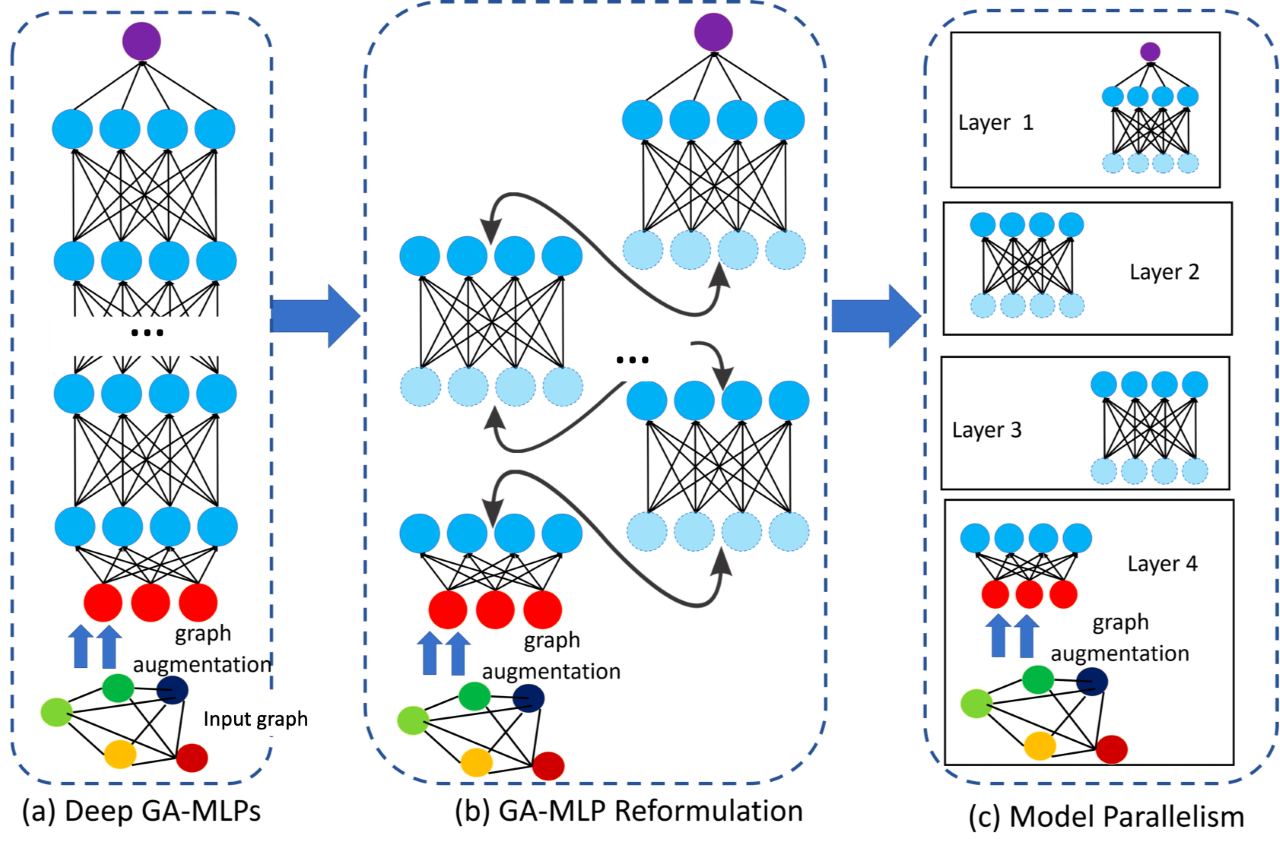}
    \caption{The overall pdADMM-G optimization algorithm: it splits GA-MLP models into layerwise components.}
    \label{fig:pdADMM-G framework}
\end{figure}
\begin{algorithm} 
\caption{The pdADMM-G Algorithm to Solve Problem \ref{prob:problem 2}} 
\begin{algorithmic}
\label{algo:distributed ADMM}
\REQUIRE $y$, $p_1=X$, $\rho$, $\nu$. 
\ENSURE $\textbf{p},\textbf{W},\textbf{b},\textbf{z},\textbf{q}$. 
\STATE Initialize $k=0$.
\WHILE{$\textbf{p}^k,\textbf{W}^{k},\textbf{b}^{k},\textbf{z}^{k},\textbf{q}^{k}$ not converged}
\STATE $p^{k+1}_l\leftarrow \arg\min\nolimits_{p_l} L_\rho(\textbf{p},\textbf{W}^k,\textbf{b}^k,\textbf{z}^k,\textbf{q}^k,\textbf{u}^k)$ for different $l$ in parallel.
\STATE $W^{k+1}_l\leftarrow \arg\min\nolimits_{W_l} L_\rho(\textbf{p}^{k+1},\textbf{W},\textbf{b}^k,\textbf{z}^{k},\textbf{q}^{k},\textbf{u}^{k})$ for different $l$ in parallel.
\STATE $b^{k+1}_l\leftarrow \arg\min\nolimits_{b_l} L_\rho(\textbf{p}^{k+1},\textbf{W}^{k+1},\textbf{b},\textbf{z}^{k},\textbf{q}^{k},\textbf{u}^{k})$ for different $l$ in parallel.
\STATE $z^{k+1}_l\leftarrow \arg\min\nolimits_{z_l} L_\rho(\textbf{p}^{k+1},\textbf{W}^{k+1},\textbf{b}^{k+1},\textbf{z},\textbf{q}^{k},\textbf{u}^{k})$ for different $l$ in parallel.
\STATE $q^{k+1}_l\leftarrow \arg\min\nolimits_{q_l} L_\rho(\textbf{p}^{k+1},\textbf{W}^{k+1},\textbf{b}^{k+1},\textbf{z}^{k+1},\textbf{q},\textbf{u}^k)$ for different $l$ in parallel.
\STATE $r^k_l\leftarrow p^{k+1}_{l+1}-q^{k+1}_l(l=1,\cdots,L)$ in parallel \# Compute residuals.
\STATE $u^{k+1}_l\leftarrow u^k_l+\rho(p^{k+1}_{l+1}-q^{k+1}_{l})$ for different $l$ in parallel.
\STATE $k\leftarrow k+1$.
\ENDWHILE
\STATE Output $\textbf{p},\textbf{W},\textbf{b},\textbf{z},\textbf{q}$.
\end{algorithmic}
\end{algorithm}
Due to space limit, the details of all subproblems are shown in Section \ref{sec:pdADMM-G subproblem} in the Appendix.\\
\indent Our proposed pdADMM-G algorithm can be efficient for training deep GA-MLP models via the greedy layerwise training strategy \cite{bengio2006greedy}. Specifically, we begin by training a shallow GA-MLP model. Next, more layers are increased to the GA-MLP model and their parameters are trained, then we introduce even more layers and iterate this process until the whole deep GA-MLP model is included. The pdADMM-G algorithm can achieve excellent performance as well as reduce training costs by this strategy.\\
\indent Last but not least, we compare the computational costs of the proposed pdADMM-G algorithm with the state-of-the-art backpropagation algorithm, on which the gradient descent is based. We show that they share the same level of computational costs. For the backpropagation algorithm, the most costly operation is the matrix multiplication $z_l=W_lp_l+b_l$ in the forward pass, where $W_l\in\mathbb{R}^{n_l\times n_{l-1}}$ and $p_l\in\mathbb{R}^{n_{l-1}\times \vert V\vert}$, which requires a time complexity of $O(n_l n_{l-1}\vert V \vert)$ \cite{cormen2022introduction}; for the proposed pdADMM-G algorithm, the most costly operation is to compute the derivative $\nabla_{W_l}\phi$, and it also involves the matrix multiplication, and hence its time complexity is again $O(n_l n_{l-1}\vert V \vert)$. However, the proposed pdADMM-G algorithm trains the whole GA-MLP model in a model parallelism fashion \cite{wang2020toward}, and therefore all computational costs can be split into different independent clients for parallel training; whereas the backpropagation algorithm is implemented sequentially, and thus it is less efficient than the proposed pdADMM-G algorithm.
\subsection{Quantization Extension of pdADMM-G (pdADMM-G-Q)}
\label{sec:pdADMM-G-q}
\indent In the proposed pdADMM-G algorithm, $p_l$ and $q_l$ are transmitted back and forth among layers (i.e. clients). However, the communication overheads of $p_l$ and $q_l$ surge for a large-scale graph $G$ with millions of nodes. To alleviate this challenge, the quantization technique is commonly utilized to reduce communication costs by mapping continuous values into a discrete set \cite{huang2021alternating}. In other words, $p_l$  is required to fit into a countable set $\Delta$, which is shown as follows:
\begin{problem}
\label{prob:problem 3}
\begin{align*}
    & \min\nolimits_{\textbf{p},\textbf{W},\textbf{b},\textbf{z},\textbf{q}} F(\textbf{p},\textbf{W},\textbf{b},\textbf{z},\textbf{q})=R(z_L;y)\\&+(\nu/2)(\sum\nolimits_{l=1}^{L}\Vert z_l-W_lp_l-b_l\Vert^2_2+\sum\nolimits_{l=1}^{L-1}\Vert q_l-f_l(z_l)\Vert^2_2),\\
    &s.t. \ p_{l+1}=q_l, \ p_{l}\in \Delta=\{\delta_1,\cdots,\delta_m\},
\end{align*}
\end{problem}
where $\delta_i(i=1,\cdots,m)\in \Delta$ are quantized values, which can be integers or  low-precision values. $m=\vert\Delta\vert$ is the cardinality of $\Delta$. To address Problem \ref{prob:problem 3}, we rewrite it into the following form:
\begin{align*}
    &\min\nolimits_{\textbf{p},\textbf{W},\textbf{b},\textbf{z},\textbf{q}} R(z_L;y)+\sum\nolimits_{l=2}^L\mathbb{I}(p_l)\\&+(\nu/2)(\sum\nolimits_{l=1}^{L}\Vert z_l-W_lp_l-b_l\Vert^2_2+\sum\nolimits_{l=1}^{L-1}\Vert q_l-f_l(z_l)\Vert^2_2),\\
    &s.t. \ p_{l+1}=q_l,
\end{align*}
where the indicator function $\mathbb{I}(p_l)$ is defined as follows: $\mathbb{I}(p_l)=0$ if $p_l\in\Delta$, and $\mathbb{I}(p_l)=+\infty$ if $p_l\not\in\Delta$. The augmented Lagrangian of Problem \ref{prob:problem 3} is shown as follows:
\begin{align*}
    \beta_\rho(\textbf{p},\textbf{W},\textbf{b},\textbf{z},\textbf{q},\textbf{u})=L_\rho(\textbf{p},\textbf{W},\textbf{b},\textbf{z},\textbf{q},\textbf{u})+\sum\nolimits_{l=2}^L \mathbb{I}(p_l),
\end{align*}
where $L_\rho$ is the augmented Lagrangian of Problem \ref{prob:problem 2}. The extended pdADMM-G-Q algorithm follows the same routine as the pdADMM-G algorithm, where $L_\rho$ is replaced with $\beta_\rho$. Due to space limit, the solutions to all subproblems generated by two proposed algorithms are shown in Section \ref{sec:pdADMM-G-Q subproblem} in the Appendix.\\
\section{Convergence Analysis}
\label{sec:convergence}
In this section, the theoretical convergence of the proposed pdADMM-G algorithm and the pdADMM-G-Q algorithm is provided. Due to space limit, we only provide sketches of proofs in this section, and their details are available in Section \ref{sec:convergence proof} in the Appendix. Our problem formulations are more difficult than existing ADMM literature: the term $\Vert q_l-f_l(z_l)\Vert^2_2$ is coupled in the objective, while it is separable in the existing ADMM formulations. To address this, we impose a mild condition that $\partial f_l(z_l)$ is bounded in Assumption \ref{ass:lipschitz continuous}, and prove that $u_l$ is controlled via $q_l$ and $z_l$ in Lemma \ref{lemma:u bound} in Section \ref{sec:convergence proof} in the Appendix.\\
 \indent Firstly, the proper function,  Lipschitz continuity, and coercivity are defined as follows:
 \begin{definition}[Proper Functions]\cite{rockafellar2009variational}.
\label{def:proper}
For a convex function $g(x):\mathbb{R}\rightarrow\mathbb{R}\bigcup\{\pm\infty\}$, $g$ is called proper if $\forall x\in \mathbb{R}$, $g(x)>-\infty$, and $\exists x_0\in\mathbb{R}$ such that $g(x_0)<+\infty$.
\end{definition}
\begin{definition}(Lipschitz Continuity)
A function $g(x)$ is Lipschitz continuous if there exists a constant $D>0$ such that $\forall x_1,x_2$, the following holds
\begin{align*}
    \Vert g(x_1)-g(x_2)\Vert\leq D\Vert x_1-x_2\Vert. 
\end{align*}
\end{definition}
\begin{definition}(Coercivity)
A function $h(x)$ is coerce over the feasible set  $\mathscr{F}$ means that $h(x)\rightarrow \infty$ if $x\in \mathscr{F}$ and $\Vert x\Vert\rightarrow \infty$.
\end{definition}
Next, the definition of a quantized stationary point \cite{huang2021alternating} is shown as follows:
\begin{definition}(Quantized Stationary Point)
The $p_l$ is a quantized stationary point of of Problem \ref{prob:problem 3} if there exists $\tau>0$ such that
\begin{align*}
    p_l\in\arg\min\nolimits_{\delta\in \Delta}\Vert \delta-(p_l-\nabla_{p_l} F(\textbf{p},\textbf{W},\textbf{b},\textbf{z},\textbf{q})/\tau)\Vert.
\end{align*}
\end{definition}
The quantized stationary point is an extension of the stationary point in the discrete setting, and any global solution $p_l$ to Problem \ref{prob:problem 3} is a quantized stationary point to Problem \ref{prob:problem 3} (Lemma 3.7 in \cite{huang2021alternating}). Then the following assumption is required for convergence analysis.
\begin{assumption}
$f_l(z_l)$ is Lipschitz continuous with coefficient $S>0$, $R(Z_L;y)$ is proper, and $F(\textbf{p},\textbf{W},\textbf{b},\textbf{z},\textbf{q})$ is coercive. Moreover, $\partial f_l(z_l)$ is bounded, i.e. there exists $M>0$ such that $\Vert\partial f_l(z_l)\Vert \leq M$. 
\label{ass:lipschitz continuous}
\end{assumption}
Assumption \ref{ass:lipschitz continuous} is mild to satisfy: most common activation functions such as Rectified Linear Unit (ReLU) \cite{wang2020toward} and leaky ReLU\cite{xu2015empirical} satisfy Assumption \ref{ass:lipschitz continuous}.
 The risk function $R(z_l;y)$ is only required to be proper, which shows that the convergence condition of our proposed pdADMM-G is milder than that of the dlADMM, which requires $R(z_l;y)$ to be Lipschitz differentiable \cite{wang2019admm}. Due to the space limit, detailed proofs are provided in Section \ref{sec:convergence proof} in the Appendix. The technical proofs follow a similar routine as dlADMM \cite{wang2019admm}. The difference consists in the fact that the dual variable $u_l$ is controlled by $q_l$ and $z_l$ (Lemma \ref{lemma:u square bound} in Section \ref{sec:convergence proof} in the Appendix), which holds under Assumption \ref{ass:lipschitz continuous}, while $u_l$ can  be controlled only by $z_l$ in the convergence proof of dlADMM. The first lemma shows that the objective keeps decreasing when $\rho$ is sufficiently large.

\begin{lemma}[Objective Decrease]
\label{lemma:objective decrease} For both the pdADMM-G algorithm and the pdADMM-G-Q algorithm, 
if $\rho>\max(4\nu S^2,(\sqrt{17}+1)\nu/2)$, there exist $C_1=\nu/2-2\nu^2S^2/\rho>0$ and $C_2=\rho/2-2\nu^2/\rho-\nu/2>0$ such that it holds for any $k\in \mathbb{N}$ that
\begin{align}
    \nonumber &L_\rho(\textbf{p}^{k},\!\textbf{W}^{k},\!\textbf{b}^k,\!\textbf{z}^{k},\!\textbf{q}^k,\!\textbf{u}^{k})\!-\!L_\rho(\textbf{p}^{k\!+\!1},\textbf{W}^{k\!+\!1},\textbf{b}^{k\!+\!1},\textbf{z}^{k\!+\!1},\textbf{q}^{k\!+\!1},\textbf{u}^{k\!+\!1})\\\nonumber &\geq \sum\nolimits_{l\!=\!2}^L (\tau^{k\!+\!1}_l/2)\Vert p^{k+1}_l\!-\!p^k_l\Vert^2_2\!+\!\sum\nolimits_{l\!=\!1}^{L}(\theta^{k\!+\!1}_l/2)\Vert  W^{k\!+\!1}_l\!-\!W^k_l\Vert^2_2\!\\&\nonumber\!+\!\sum\nolimits_{l\!=\!1}^{L}(\nu/2)\Vert  b^{k\!+\!1}_l\!-\!b^k_l\Vert^2_2\!+\!\sum\nolimits_{l\!=\!1}^{L\!-\!1} C_1\Vert z^{k+1}_l\!-\!z^k_l\Vert^2_2\\&\!+\!(\nu/2)\Vert z^{k+1}_L-z^k_L\Vert^2_2+\sum\nolimits_{l=1}^{L-1}C_2\Vert q^{k+1}_l-q^k_l\Vert^2_2,
    \label{eq:pdADMM-G objective decrease}
\\\nonumber &\beta_\rho(\textbf{p}^{k},\!\textbf{W}^{k},\!\textbf{b}^k,\!\textbf{z}^{k},\!\textbf{q}^k,\!\textbf{u}^{k})\!-\!\beta_\rho(\textbf{p}^{k\!+\!1},\textbf{W}^{k\!+\!1},\textbf{b}^{k\!+\!1},\textbf{z}^{k\!+\!1},\textbf{q}^{k\!+\!1},\textbf{u}^{k\!+\!1})\\\nonumber &\geq \!\sum\nolimits_{l\!=\!1}^{L}(\theta^{k\!+\!1}_l/2)\Vert  W^{k\!+\!1}_l\!-\!W^k_l\Vert^2_2\!+\!\sum\nolimits_{l\!=\!1}^{L}(\nu/2)\Vert  b^{k\!+\!1}_l\!-\!b^k_l\Vert^2_2\!\\\nonumber&+\!\sum\nolimits_{l\!=\!1}^{L\!-\!1} C_1\Vert z^{k+1}_l\!-\!z^k_l\Vert^2_2+\!(\nu/2)\Vert z^{k+1}_L-z^k_L\Vert^2_2\\&+\sum\nolimits_{l=1}^{L-1}C_2\Vert q^{k+1}_l-q^k_l\Vert^2_2.
    \label{eq:pdADMM-G-q objective decrease}
\end{align}
\end{lemma}
\begin{sproof}
\indent 
They can be proven via the optimality conditions of all subproblems, and Assumption \ref{ass:lipschitz continuous}.
\end{sproof}
\indent Lemma \ref{lemma:lower bounded} shows that the objective is bounded from below when $\rho$ is large enough, and all variables are bounded.
\begin{lemma} [Bounded Objective]
\label{lemma:lower bounded}

(1). For the pdADMM-G algorithm, if $\rho> \nu$, then
 $L_\rho(\textbf{p}^k,\textbf{W}^{k},\textbf{b}^{k},\textbf{z}^{k},\textbf{q}^{k},\textbf{u}^{k})$
 is lower bounded. Moreover, $\textbf{p}^k,\textbf{W}^k,\textbf{b}^k,\textbf{z}^k,\textbf{q}^{k}$,and $\textbf{u}^k$ are bounded, i.e. there exist $\mathbb{N}_\textbf{p}$, $\mathbb{N}_\textbf{W}$, $\mathbb{N}_\textbf{b}$,  $\mathbb{N}_\textbf{z}$,  $\mathbb{N}_\textbf{q}$, and $\mathbb{N}_\textbf{u}>0$, such that $\Vert \textbf{p}^k\Vert\leq \mathbb{N}_\textbf{p}$, $\Vert \textbf{W}^k\Vert\leq \mathbb{N}_\textbf{W}$, $\Vert \textbf{b}^k\Vert\leq \mathbb{N}_\textbf{b}$, $\Vert \textbf{z}^k\Vert\leq \mathbb{N}_\textbf{z}$, $\Vert \textbf{q}^k\Vert\leq \mathbb{N}_\textbf{q}$, and $\Vert \textbf{u}^k\Vert\leq \mathbb{N}_\textbf{u}$. \\
 (2). For the pdADMM-G-Q algorithm, if $\rho> \nu$, then
 $\beta_\rho(\textbf{p}^k,\textbf{W}^{k},\textbf{b}^{k},\textbf{z}^{k},\textbf{q}^{k},\textbf{u}^{k})$
 is lower bounded. Moreover, $\textbf{W}^k,\textbf{b}^k,\textbf{z}^k,\textbf{q}^{k}$,and $\textbf{u}^k$ are bounded, i.e. there exist  $\overline{\mathbb{N}}_\textbf{W}$, $\overline{\mathbb{N}}_\textbf{b}$,  $\overline{\mathbb{N}}_\textbf{z}$,  $\overline{\mathbb{N}}_\textbf{q}$, and $\overline{\mathbb{N}}_\textbf{u}>0$, such that $\Vert \textbf{W}^k\Vert\leq \overline{\mathbb{N}}_\textbf{W}$, $\Vert \textbf{b}^k\Vert\leq \overline{\mathbb{N}}_\textbf{b}$, $\Vert \textbf{z}^k\Vert\leq \overline{\mathbb{N}}_\textbf{z}$, $\Vert \textbf{q}^k\Vert\leq \overline{\mathbb{N}}_\textbf{q}$, and $\Vert \textbf{u}^k\Vert\leq \overline{\mathbb{N}}_\textbf{u}$.\end{lemma}
 \begin{sproof}
 \indent We only show the sketch proof of (1) because (2) follows the same routine as (1).
 In order to prove the boundness of  $L_\rho$, we should prove the following:
\begin{align*}
     &L_\rho(\textbf{p}^k,\textbf{W}^k,\textbf{b}^k,\textbf{z}^k,\textbf{q}^{k},\textbf{u}^k)\\&\geq F(\textbf{p}^k,\textbf{W}^k,\textbf{b}^k,\textbf{z}^k,\textbf{q}^{'})+((\rho-\nu)/2)\Vert p^k_{l+1}-q^k_l\Vert^2_2\\&> -\infty,
\end{align*}
where $p^k_{l+1}=q^{'}_l$. Therefore, $F(\textbf{p}^k,\textbf{W}^k,\textbf{b}^k,\textbf{z}^k,\textbf{q}^{'})$ and $((\rho-\nu)/2)\Vert p^k_{l+1}-q^k_l\Vert^2_2$ are upper bounded by $L_\rho(\textbf{p}^k,\textbf{W}^k,\textbf{b}^k,\textbf{z}^k,\textbf{q}^{k},\textbf{u}^k)$ and hence $L_\rho(\textbf{p}^0,\textbf{W}^0,\textbf{b}^0,\textbf{z}^0,\textbf{q}^{0},\textbf{u}^0)$ (Lemma \ref{lemma:objective decrease}). The boundness of variables can be obtained via
Assumption \ref{ass:lipschitz continuous}.
\end{sproof}
\vspace{-0.5cm}
 \indent Based on Lemmas \ref{lemma:objective decrease} and \ref{lemma:lower bounded}, the following theorem ensures that the objective  is convergent.
 \begin{theorem}[Convergent Objective] \label{theo:convergent variable}
(1). For the pdADMM-G algorithm, if $\rho>\max(4\nu S^2,(\sqrt{17}+1)\nu/2)$, then
$L_\rho(\textbf{p}^k,\textbf{W}^k,\textbf{b}^k,\textbf{z}^k,\textbf{q}^{k},\textbf{u}^k)$ is convergent. Moreover, $\lim_{k\rightarrow\infty}\Vert\textbf{p}^{k+1}-\textbf{p}^{k}\Vert^2_2=0$, $\lim_{k\rightarrow\infty}\Vert\textbf{W}^{k+1}-\textbf{W}^{k}\Vert^2_2=0$, $\lim_{k\rightarrow\infty}\Vert\textbf{b}^{k+1}-\textbf{b}^{k}\Vert^2_2=0$, $\lim_{k\rightarrow\infty}\Vert\textbf{z}^{k+1}-\textbf{z}^{k}\Vert^2_2=0$, $\lim_{k\rightarrow\infty}\Vert\textbf{q}^{k+1}-\textbf{q}^{k}\Vert^2_2=0$, $\lim_{k\rightarrow\infty}\Vert\textbf{u}^{k+1}-\textbf{u}^{k}\Vert^2_2=0$.\\
(2). For the pdADMM-G-Q algorithm, if $\rho>\max(4\nu S^2,(\sqrt{17}+1)\nu/2)$, then
$\beta_\rho(\textbf{p}^k,\textbf{W}^k,\textbf{b}^k,\textbf{z}^k,\textbf{q}^{k},\textbf{u}^k)$ is convergent. Moreover, $\lim_{k\rightarrow\infty}\Vert\textbf{W}^{k+1}-\textbf{W}^{k}\Vert^2_2=0$, $\lim_{k\rightarrow\infty}\Vert\textbf{b}^{k+1}-\textbf{b}^{k}\Vert^2_2=0$, $\lim_{k\rightarrow\infty}\Vert\textbf{z}^{k+1}-\textbf{z}^{k}\Vert^2_2=0$, $\lim_{k\rightarrow\infty}\Vert\textbf{q}^{k+1}-\textbf{q}^{k}\Vert^2_2=0$, $\lim_{k\rightarrow\infty}\Vert\textbf{u}^{k+1}-\textbf{u}^{k}\Vert^2_2=0$.
\end{theorem}
\begin{sproof}
 This theorem can be derived by taking the limit on both sides of Inequality \eqref{eq:pdADMM-G objective decrease}. 
\end{sproof}
 \indent The third lemma guarantees that the subgradient of the objective is upper bounded, which is stated as follows:
 \begin{lemma}[Bounded Subgradient]
\label{lemma:subgradient bound}
(1). For the pdADMM-G algorithm, there exists a constant $C>0$ and $g^{k+1}\in \partial L_\rho(\textbf{p}^{k+1},\textbf{W}^{k+1},\textbf{b}^{k+1},\textbf{z}^{k+1},\textbf{q}^{k+1},\textbf{u}^{k+1})$ such that
\begin{align*}
    &\Vert g^{k+1}\Vert \leq C(\Vert\textbf{p}^{k+1}-\textbf{p}^{k}\Vert+\Vert\textbf{W}^{k+1}-\textbf{W}^{k}\Vert+\Vert\textbf{b}^{k+1}-\textbf{b}^{k}\Vert\\&+\Vert\textbf{z}^{k+1}-\textbf{z}^{k}\Vert+\Vert\textbf{q}^{k+1}-\textbf{q}^{k}\Vert+\Vert\textbf{u}^{k+1}-\textbf{u}^{k}\Vert).
\end{align*}
(2). For the pdADMM-G-Q algorithm, there exists a constant $\overline{C}>0$, $\overline{g}_\textbf{W}^{k+1}\in \nabla_{\textbf{W}^{k+1}} \beta_\rho(\textbf{p}^{k+1},\textbf{W}^{k+1},\textbf{b}^{k+1},\textbf{z}^{k+1},\textbf{q}^{k+1},\textbf{u}^{k+1})$, $\overline{g}_\textbf{b}^{k+1}\in \nabla_{\textbf{b}^{k+1}} \beta_\rho(\textbf{p}^{k+1},\textbf{W}^{k+1},\textbf{b}^{k+1},\textbf{z}^{k+1},\textbf{q}^{k+1},\textbf{u}^{k+1})$, $\overline{g}_\textbf{z}^{k+1}\in \partial_{\textbf{z}^{k+1}} \beta_\rho(\textbf{p}^{k+1},\textbf{W}^{k+1},\textbf{b}^{k+1},\textbf{z}^{k+1},\textbf{q}^{k+1},\textbf{u}^{k+1})$, $\overline{g}_\textbf{q}^{k+1}\in \nabla_{\textbf{q}^{k+1}} \beta_\rho(\textbf{p}^{k+1},\textbf{W}^{k+1},\textbf{b}^{k+1},\textbf{z}^{k+1},\textbf{q}^{k+1},\textbf{u}^{k+1})$, $\overline{g}_\textbf{u}^{k+1}\in \nabla_{\textbf{u}^{k+1}} \beta_\rho(\textbf{p}^{k+1},\textbf{W}^{k+1},\textbf{b}^{k+1},\textbf{z}^{k+1},\textbf{q}^{k+1},\textbf{u}^{k+1})$ such that
\begin{align*}
    &\Vert \overline{g}_\textbf{W}^{k+1}\Vert \leq \overline{C}(\Vert\textbf{W}^{k+1}-\textbf{W}^{k}\Vert+\Vert\textbf{b}^{k+1}-\textbf{b}^{k}\Vert\\&+\Vert\textbf{z}^{k+1}-\textbf{z}^{k}\Vert+\Vert\textbf{q}^{k+1}-\textbf{q}^{k}\Vert+\Vert\textbf{u}^{k+1}-\textbf{u}^{k}\Vert),\\&
    \Vert \overline{g}_\textbf{b}^{k+1}\Vert \leq \overline{C}(\Vert\textbf{W}^{k+1}-\textbf{W}^{k}\Vert+\Vert\textbf{b}^{k+1}-\textbf{b}^{k}\Vert\\&+\Vert\textbf{z}^{k+1}-\textbf{z}^{k}\Vert+\Vert\textbf{q}^{k+1}-\textbf{q}^{k}\Vert+\Vert\textbf{u}^{k+1}-\textbf{u}^{k}\Vert),\\&\Vert \overline{g}_\textbf{z}^{k+1}\Vert \leq \overline{C}(\Vert\textbf{W}^{k+1}-\textbf{W}^{k}\Vert+\Vert\textbf{b}^{k+1}-\textbf{b}^{k}\Vert\\&+\Vert\textbf{z}^{k+1}-\textbf{z}^{k}\Vert+\Vert\textbf{q}^{k+1}-\textbf{q}^{k}\Vert+\Vert\textbf{u}^{k+1}-\textbf{u}^{k}\Vert),\\&
    \Vert \overline{g}_\textbf{q}^{k+1}\Vert \leq \overline{C}(\Vert\textbf{W}^{k+1}-\textbf{W}^{k}\Vert+\Vert\textbf{b}^{k+1}-\textbf{b}^{k}\Vert\\&+\Vert\textbf{z}^{k+1}-\textbf{z}^{k}\Vert+\Vert\textbf{q}^{k+1}-\textbf{q}^{k}\Vert+\Vert\textbf{u}^{k+1}-\textbf{u}^{k}\Vert),\\&\Vert \overline{g}_\textbf{u}^{k+1}\Vert \leq \overline{C}(\Vert\textbf{W}^{k+1}-\textbf{W}^{k}\Vert+\Vert\textbf{b}^{k+1}-\textbf{b}^{k}\Vert\\&+\Vert\textbf{z}^{k+1}-\textbf{z}^{k}\Vert+\Vert\textbf{q}^{k+1}-\textbf{q}^{k}\Vert+\Vert\textbf{u}^{k+1}-\textbf{u}^{k}\Vert).
\end{align*}
\end{lemma}
\begin{sproof}
To prove this lemma, the subgradient is proven to be upper bounded by the linear combination of $\Vert\textbf{p}^{k+1}-\textbf{p}^{k}\Vert$, $ \Vert\textbf{W}^{k+1}-\textbf{W}^{k}\Vert$,  $\Vert\textbf{b}^{k+1}-\textbf{b}^{k}\Vert$, $\Vert\textbf{z}^{k+1}-\textbf{z}^{k}\Vert$, $\Vert\textbf{q}^{k+1}-\textbf{q}^{k}\Vert$, and $\Vert\textbf{u}^{k+1}-\textbf{u}^{k}\Vert$.
\end{sproof}
 \indent Now based on Theorem \ref{theo:convergent variable}, and Lemma \ref{lemma:subgradient bound}, the convergence of the pdADMM-G algorithm to a stationary point is presented in the following theorem.
\begin{theorem} [Convergence of the pdADMM-G algorithm]
\label{theo: pdADMM-G global convergence}
If $\rho>\max(4\nu S^2,(\sqrt{17}+1)\nu/2)$, then for the variables $(\textbf{p},\textbf{W},\textbf{b},\textbf{z},\textbf{q},\textbf{u})$ in Problem \ref{prob:problem 2}, starting from any $(\textbf{p}^{0},\textbf{W}^{0},\textbf{b}^{0},\textbf{z}^{0},\textbf{q}^{0},\textbf{u}^{0})$, $(\textbf{p}^{k},\textbf{W}^{k},\textbf{b}^{k},\textbf{z}^{k},\textbf{q}^{k},\textbf{u}^{k})$ has at least a limit point $(\textbf{p}^*,\textbf{W}^*,\textbf{b}^*,\textbf{z}^*,\textbf{q}^*,\textbf{u}^*)$, and any limit point is a stationary point of Problem \ref{prob:problem 2}. That is, $0\in \partial L_\rho(\textbf{p}^*,\textbf{W}^*,\textbf{b}^*,\textbf{z}^*,\textbf{q}^*,\textbf{u}^*)$. In other words, 
\begin{align*}
    & p^*_{l+1}=q^*_l, \ \nabla_{\textbf{p}^*} L_\rho(\textbf{p}^*,\textbf{W}^*,\textbf{b}^*,\textbf{z}^*,\textbf{q}^*,\textbf{u}^*)\!=\!0,\\& \
    \nabla_{\textbf{W}^*}\! L_\rho\!(\textbf{p}^*\!,\!\textbf{W}^*\!,\!\textbf{b}^*,\textbf{z}^*,\textbf{q}^*,\textbf{u}^*)\!=\!0, \nabla_{\textbf{b}^*}\! L_\rho\!(\textbf{p}^*\!,\!\textbf{W}^*\!,\!\textbf{b}^*,\textbf{z}^*,\textbf{q}^*,\textbf{u}^*)\!=\!0, \\& 0\!\in\!\partial_{\textbf{z}^*}\! L_\rho\!(\textbf{p}^*,\!\textbf{W}^*,\textbf{b}^*,\textbf{z}^*,\textbf{q}^*,\textbf{u}^*),\!
     \nabla_{\textbf{q}^*} L_\rho\!(\textbf{p}^*,\!\textbf{W}^*,\textbf{b}^*,\textbf{z}^*,\textbf{q}^*,\textbf{u}^*)\!=\!0.
\end{align*}
\end{theorem}
\begin{sproof}
 This theorem can be derived directly from Lemma \ref{lemma:lower bounded} and Lemma \ref{lemma:subgradient bound}.
\end{sproof}
\indent Theorem \ref{theo: pdADMM-G global convergence} shows that our proposed pdADMM-G algorithm converges for sufficiently large $\rho$, which is consistent with previous literature \cite{wang2019admm}. Similarly, the convergence of the proposed pdADMM-G-Q algorithm is shown as follows:
\begin{theorem} [Convergence of the pdADMM-G-Q algorithm]
\label{theo: pdADMM-G-q global convergence}
If $\rho>\max(4\nu S^2,(\sqrt{17}+1)\nu/2)$, then for the variables $(\textbf{p},\textbf{W},\textbf{b},\textbf{z},\textbf{q},\textbf{u})$ in Problem \ref{prob:problem 3}, starting from any $(\textbf{p}^{0},\textbf{W}^{0},\textbf{b}^{0},\textbf{z}^{0},\textbf{q}^{0},\textbf{u}^{0})$, $(\textbf{p}^{k},\textbf{W}^{k},\textbf{b}^{k},\textbf{z}^{k},\textbf{q}^{k},\textbf{u}^{k})$ has at least a limit point $(\textbf{p}^*,\textbf{W}^*,\textbf{b}^*,\textbf{z}^*,\textbf{q}^*,\textbf{u}^*)$, and any limit point $(\textbf{W}^*,\textbf{b}^*,\textbf{z}^*,\textbf{q}^*,\textbf{u}^*)$ is a stationary point of Problem \ref{prob:problem 3}. Moreover, if $\tau^{k+1}_l$ is bounded, then $\textbf{p}^*$ is a quantized stationary point of Problem \ref{prob:problem 3}. That is 
\begin{align*}
    & p^*_{l+1}=q^*_l, \
    \nabla_{\textbf{W}^*} \beta_\rho(\textbf{p}^*,\textbf{W}^*,\textbf{b}^*,\textbf{z}^*,\textbf{q}^*,\textbf{u}^*)=0,\\& \nabla_{\textbf{b}^*} \beta_\rho(\textbf{p}^*,\textbf{W}^*,\textbf{b}^*,\textbf{z}^*,\textbf{q}^*,\textbf{u}^*)=0, \\& 0\in\partial_{\textbf{z}^*} \beta_\rho(\textbf{p}^*,\textbf{W}^*,\textbf{b}^*,\textbf{z}^*,\textbf{q}^*,\textbf{u}^*),\\& 
    \nabla_{\textbf{q}^*} \beta_\rho(\textbf{p}^*,\textbf{W}^*,\textbf{b}^*,\textbf{z}^*,\textbf{q}^*,\textbf{u}^*)=0,\\&
    p^*_l\in\arg\min\nolimits_{\delta\in \Delta}\Vert \delta-(p^*_l-\nabla_{p^*_l} F(\textbf{p}^*,\textbf{W}^*,\textbf{b}^*,\textbf{z}^*,\textbf{q}^*)/\tau^*_l)\Vert.
\end{align*}
where $\tau^*_l$ is a limit point of $\tau^k_l$.
\end{theorem}
\begin{sproof}
This theorem is proven using a similar procedure as Theorem \ref{theo: pdADMM-G global convergence}, and the definition of the quantized stationary point.
\end{sproof}
The only difference between Theorems \ref{theo: pdADMM-G global convergence} and \ref{theo: pdADMM-G-q global convergence} is that $\textbf{p}^*$ is a stationary point in Problem \ref{prob:problem 2} and a quantized stationary point in Problem \ref{prob:problem 3}.  Next, the following theorem ensures the sublinear convergence rate $o(1/k)$ of the proposed pdADMM-G algorithm and the pdADMM-G-Q algorithm.
\begin{theorem}[Convergence Rate]
(1). For the pdADMM-G algorithm and a sequence $(\textbf{p}^k,\textbf{W}^k,\textbf{b}^k,\textbf{z}^k,\textbf{q}^k,\textbf{u}^k)$, define $c_k=\min\nolimits_{0\leq i\leq k}(\sum\nolimits_{l=2}^L (\tau^{i+1}_l/2)\Vert p^{i+1}_l\!-\!p^i_l\Vert^2_2+\sum\nolimits_{l\!=\!1}^{L}(\theta^{i\!+\!1}_l/2)\Vert  W^{i\!+\!1}_l\!-\!W^i_l\Vert^2_2+\sum\nolimits_{l\!=\!1}^{L}(\nu/2)\Vert  b^{i\!+\!1}_l\!-\!b^i_l\Vert^2_2+\sum\nolimits_{l=1}^{L-1} C_1\Vert z^{i+1}_l-z^i_l\Vert^2_2+(\nu/2)\Vert z^{i+1}_L-z^i_L\Vert^2_2+\sum\nolimits_{l=1}^{L-1}C_2\Vert q^{i+1}_l-q^i_l\Vert^2_2)$ where $C_1=\nu/2-2\nu^2S^2/\rho>0$ and $C_2=\rho/2-2\nu^2/\rho-\nu/2>0$, then the convergence rate of $c_k$ is $o(1/k)$.\\
(2). For the pdADMM-G-Q algorithm and a sequence $(\textbf{W}^k,\textbf{b}^k,\textbf{z}^k,\textbf{q}^k,\textbf{u}^k)$, define $d_k=\min\nolimits_{0\leq i\leq k}(\sum\nolimits_{l\!=\!1}^{L}(\theta^{i\!+\!1}_l/2)\Vert  W^{i\!+\!1}_l\!-\!W^i_l\Vert^2_2+\sum\nolimits_{l\!=\!1}^{L}(\nu/2)\Vert  b^{i\!+\!1}_l\!-\!b^i_l\Vert^2_2+\sum\nolimits_{l=1}^{L-1} C_1\Vert z^{i+1}_l-z^i_l\Vert^2_2+(\nu/2)\Vert z^{i+1}_L-z^i_L\Vert^2_2+\sum\nolimits_{l=1}^{L-1}C_2\Vert q^{i+1}_l-q^i_l\Vert^2_2)$ where $C_1=\nu/2-2\nu^2S^2/\rho>0$ and $C_2=\rho/2-2\nu^2/\rho-\nu/2>0$, then the convergence rate of $d_k$ is $o(1/k)$.
\label{theo: convergence rate}
\end{theorem}
\begin{sproof}
 (1). In order to prove the convergence rate $o(1/k)$, $c_k$ satisfies three conditions: (a) $c_k\geq c_{k+1}$,  (b) $\sum\nolimits_{k=0}^\infty c_k$ is bounded, and (c) $c_k\geq0$. \\
 (2). $d_k$ can be proven using a similar procedure as (1).
\end{sproof}


\section{Experiments}
\label{sec:experiments}
\indent In this section, we evaluate the performance of the proposed pdADMM-G algorithm and the proposed pdADMM-G-Q algorithm on GA-MLP models using nine benchmark datasets. Convergence and computational overheads are demonstrated on different datasets. Speedup and test performance are compared with several state-of-the-art optimizers. All experiments were conducted on the Amazon Web Services (AWS) p2.16xlarge instance, with 16 NVIDIA K80 GPUs, 64vCPUs, a processor Intel(R) Xeon(R) CPU E5-2686 v4 @ 2.30GHz, and 732GiB of RAM.

\begin{table}
\tiny
    \centering
 
    \begin{tabular}{m{0.8cm}|m{0.7cm}|m{0.8cm}|m{0.3cm}|m{0.4cm}|m{0.6cm}|m{0.6cm}|m{0.6cm}}
    \hline\hline
         Dataset&\tabincell{c}{Node\#}&\tabincell{c}{Edge\#}&\tabincell{c}{Class\#}&
{Feature\#}&\tabincell{c}{Training\\Set\#}&\tabincell{c}{Validation\\ Set\#}&\tabincell{c}{Test\\Set\#}\\\hline 
\tabincell{c}{Cora} &\num{ 2485}& \num{10556} & \num{7} & \num{1433} &\num{140}&\num{500}&\num{1000}\\\hline 
\tabincell{c}{PubMed}&\num{19717}  & \num{88648} &  \num{3}  & \num{500} &\num{60}&\num{500}&\num{1000}\\\hline 
\tabincell{c}{Citeseer}&\num{2110}& \num{9104}  & \num{6} & \num{3703} &\num{120}&\num{500}&\num{1000}\\\hline 
\tabincell{c}{Amazon  \\ Computers}&\num{13381}&\num{491722} &\num{10}&\num{767}&\num{200}&\num{1000}&\num{1000}\\\hline 
\tabincell{c}{Amazon \\ Photo}&\num{7487}& \num{238162} &\num{8}&\num{745}&\num{160}&\num{1000}&\num{1000}\\\hline 
\tabincell{c}{Coauthor  \\ CS}&\num{18333} & \num{163788} & \num{15} & \num{6805} &\num{300}&\num{1000}&\num{1000}\\\hline 
\tabincell{c}{Coauthor  \\ Physics} &\num{34493}& \num{495924} & \num{5} & \num{8415} &\num{100}&\num{1000}&\num{1000}\\\hline 
\tabincell{c}{Flickr}&\num{89250}&\num{899756}&\num{7}&\num{500}&\num{44625}&\num{22312}&\num{22313}\\\hline
\tabincell{c}{Ogbn-Arxiv}&\num{ 169343}& \num{1166243} &\num{40}&\num{128}&\num{90941}&\num{29799}&\num{48603}\\\hline\hline
    \end{tabular}
 \caption{Dataset statistics.}
 \vspace{-0.5cm}
    \label{tab:dataset}
\end{table}

\subsection{Datasets and Settings}
 \indent  Nine benchmark datasets were used for experimental evaluation, whose statistics are shown in Table \ref{tab:dataset}. Each dataset is split into a training set, a validation set, and a test set. Due to space limit, their details can be found in Section \ref{sec:dataset} in the Appendix.\\
\indent When it comes to experimental settings,  we set $K=4$ for the multi-hop operator $\Psi$, and defined a diagonal degree matrix $D$ where $D_{ii}=\sum\nolimits_{j=1}^{\vert V\vert} A_{ij}$, and a renormalized adjacency matrix $\tilde{A}=(D+I)^{-1/2}(A+I)(D+I)^{-1/2}\in\mathbb{R}^{\vert V\vert\times\vert V\vert}$\cite{kipf2017semi}. Moreover, we set $\Psi=\{I,\tilde{A},\tilde{A}^2,\tilde{A}^3\}$ \cite{chen2021graph}. For all GA-MLP models, the activation function was set to ReLU. The loss function was set to the cross-entropy loss. For the pdADMM-G-Q algorithm, $\Delta=\{-1,0,1,\cdots,20\} $ in Problem \ref{prob:problem 3}, and $\textbf{p}$ was quantized by default.

\subsection{Comparison Methods}
\indent GD and its variants are state-of-the-art optimizers and hence served as comparison methods. For GD-based methods, all datasets were used for training models in a full-batch fashion. All hyperparameters were chosen by maximizing the performance of validation sets. Due to space limit, hyperparameter settings of all methods are shown in Section \ref{sec:hyperparameter} in the Appendix. The following are their brief introductions: \\
\indent 1. Gradient Descent (GD) \cite{bottou2010large}. The GD and its variants are the most popular deep learning optimizers. The GD updates parameters simply based on their gradients. \\
\indent 2. Adaptive learning rate method (Adadelta) \cite{zeiler2012adadelta}. The Adadelta is proposed to overcome the sensitivity to hyperparameter selection. \\
\indent 3. Adaptive gradient algorithm (Adagrad) \cite{duchi2011adaptive}. Adagrad is an improved version of GD: rather than fixing the learning rate during training, it adapts the learning rate to the hyperparameter. \\
\indent 4. Adaptive momentum estimation (Adam) \cite{kingma2014adam}. Adam is the most popular optimization method for deep learning models. It estimates the first and second momentum in order to correct the biased gradient, and thus accelerates empirical convergence.

 \begin{figure*}
  \begin{minipage}{0.24\linewidth}
  \centerline{\includegraphics[width=0.9\linewidth]{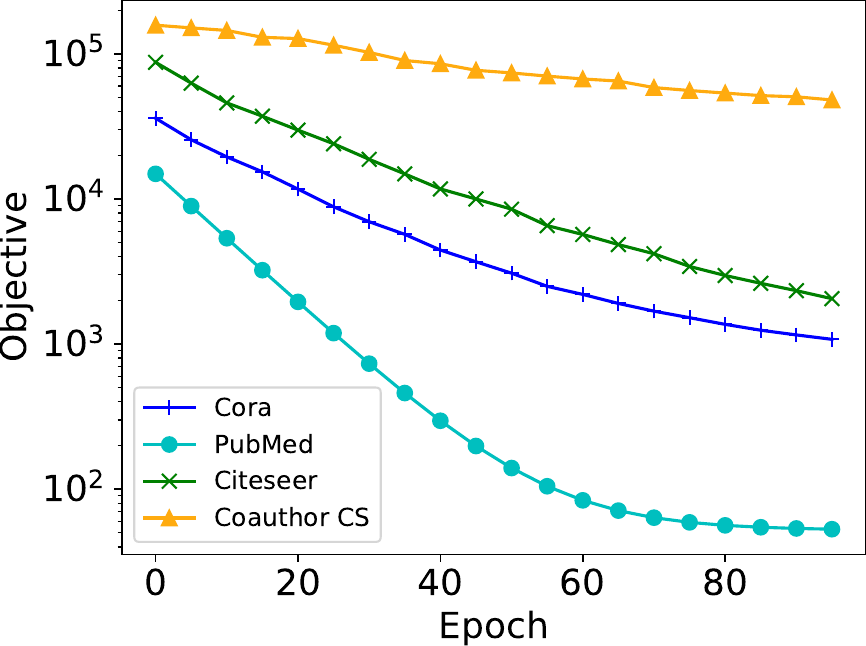}}
  \centerline{(a). pdADMM-G Objective.}
  \end{minipage}
  \hfill
  \begin{minipage}{0.24\linewidth}
  \centerline{\includegraphics[width=0.9\linewidth]{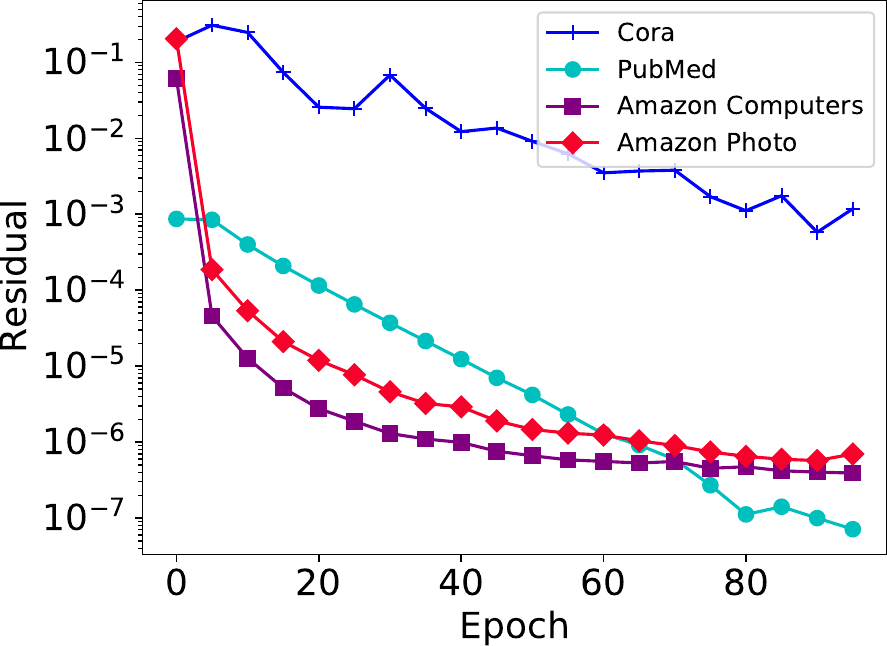}}
  \centerline{(b). pdADMM-G Residual.}
   \end{minipage}
   \hfill
  \begin{minipage}{0.24\linewidth}
  \centerline{\includegraphics[width=0.9\linewidth]{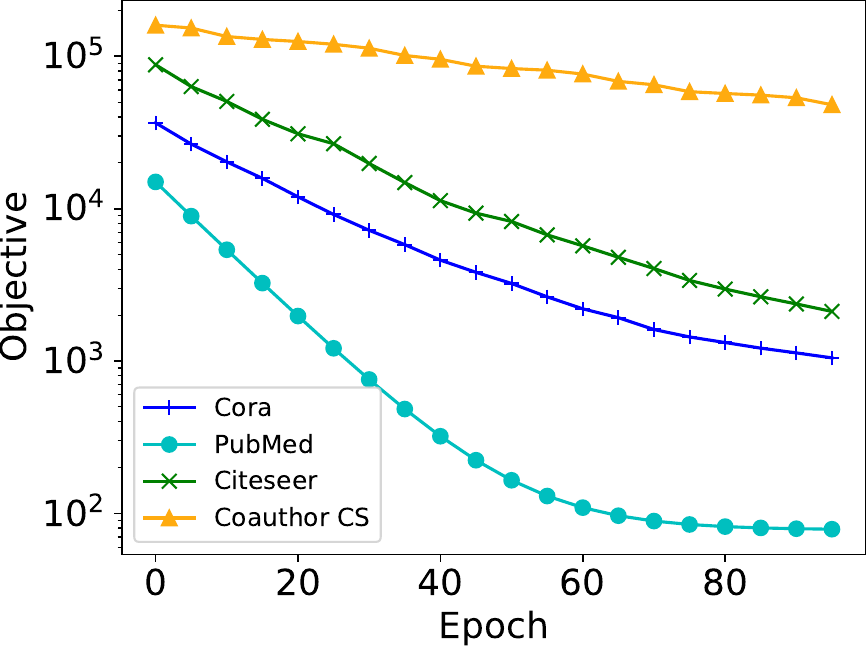}}
  \centerline{(c). pdADMM-G-Q Objective.}
   \end{minipage}
   \hfill
  \begin{minipage}{0.24\linewidth}
  \centerline{\includegraphics[width=0.9\linewidth]{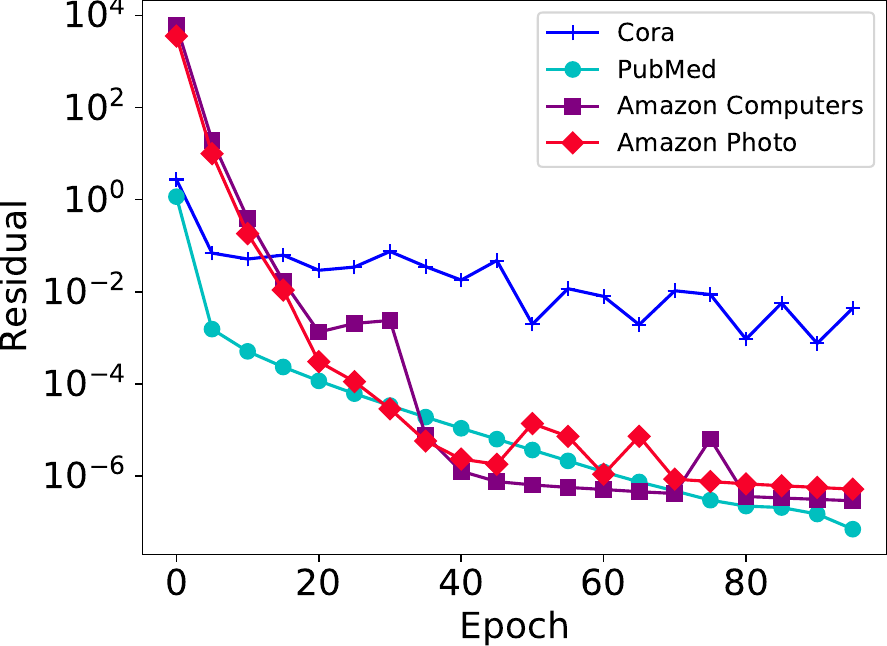}}
  \centerline{(d). pdADMM-G-Q Residual.}
   \end{minipage}
   \caption{Convergence curves of the pdADMM-G algorithm and the pdADMM-G-Q algorithm in four datasets: they both converge.}
 \label{fig:convergence}
 \vspace{-0.5cm}
\end{figure*}
\vspace{-0.5cm}
\subsection{Convergence}
\indent Firstly, in order to validate the convergence of two proposed algorithms, we set up a GA-MLP model with $10$ layers, each of which has $\num{1000}$ neurons. The number of epochs was set to $100$. $\nu$ and $\rho$ were set to $0.01$ and $1$, respectively.\\ 
\indent Figure \ref{fig:convergence} demonstrates objectives and residuals of two proposed algorithms on four datasets. Overall, the objectives and residuals of the two proposed algorithms are convergent. From Figure \ref{fig:convergence}(a) and Figure \ref{fig:convergence}(c), the objectives of the two proposed algorithms decrease drastically at the first $50$ epochs and then drop smoothly to the end. The objectives on the PubMed dataset achieve the lowest among all four datasets, whereas these on the Coauthor CS dataset are the highest, which still reach near $10^5$ at the $100$-th epoch. As for residuals, even though the residuals of the pdADMM-G-Q algorithm are higher than these of the pdADMM-G algorithm initially, they both converge sublinearly to 0, which is consistent with Theorem \ref{theo: pdADMM-G global convergence} and Theorem \ref{theo: pdADMM-G-q global convergence}.  Specifically, as shown in Figure \ref{fig:convergence}(b) and Figure \ref{fig:convergence}(d), the residuals on the Cora dataset decrease more slowly with 
fluctuation than these on other datasets, while residuals on the Amazon Computers and Amazon Photo datasets demonstrate the fastest decreasing speed at the first $40$ epochs before reaching a plateau less than $10^{-6}$. The residuals on the PubMed dataset accomplish the lowest values among all four datasets again with a value of less than $10^{-7}$. 
\begin{figure}
\begin{minipage}{0.48\linewidth}
  \centering
  \centerline{\includegraphics[width=0.95\linewidth]{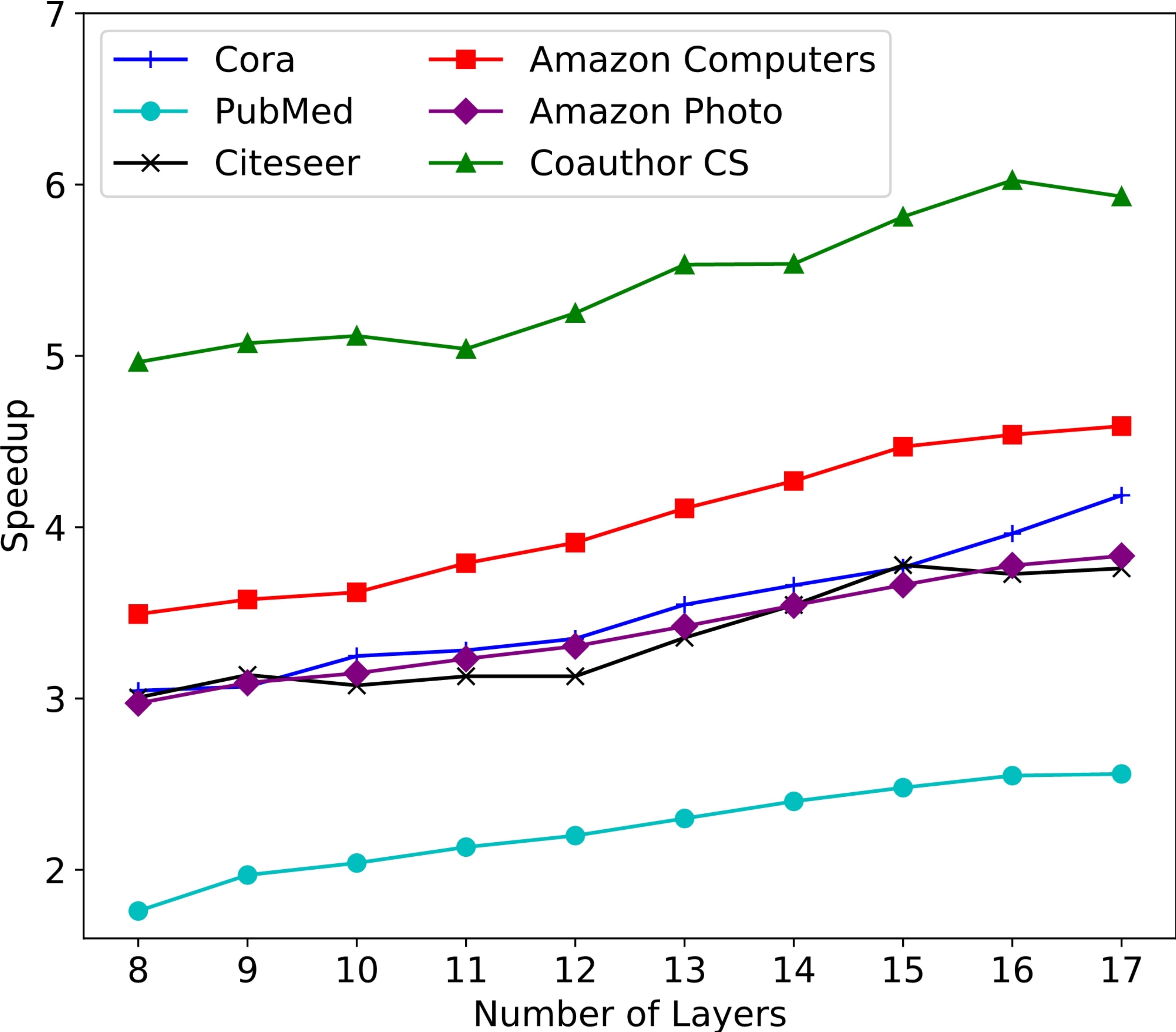}}
  \centerline{(a). The speedup}\centerline{on small datasets.}
\end{minipage}
\hfill
\begin{minipage}{0.48\linewidth}
  \centering
  \centerline{\includegraphics[width=\linewidth]{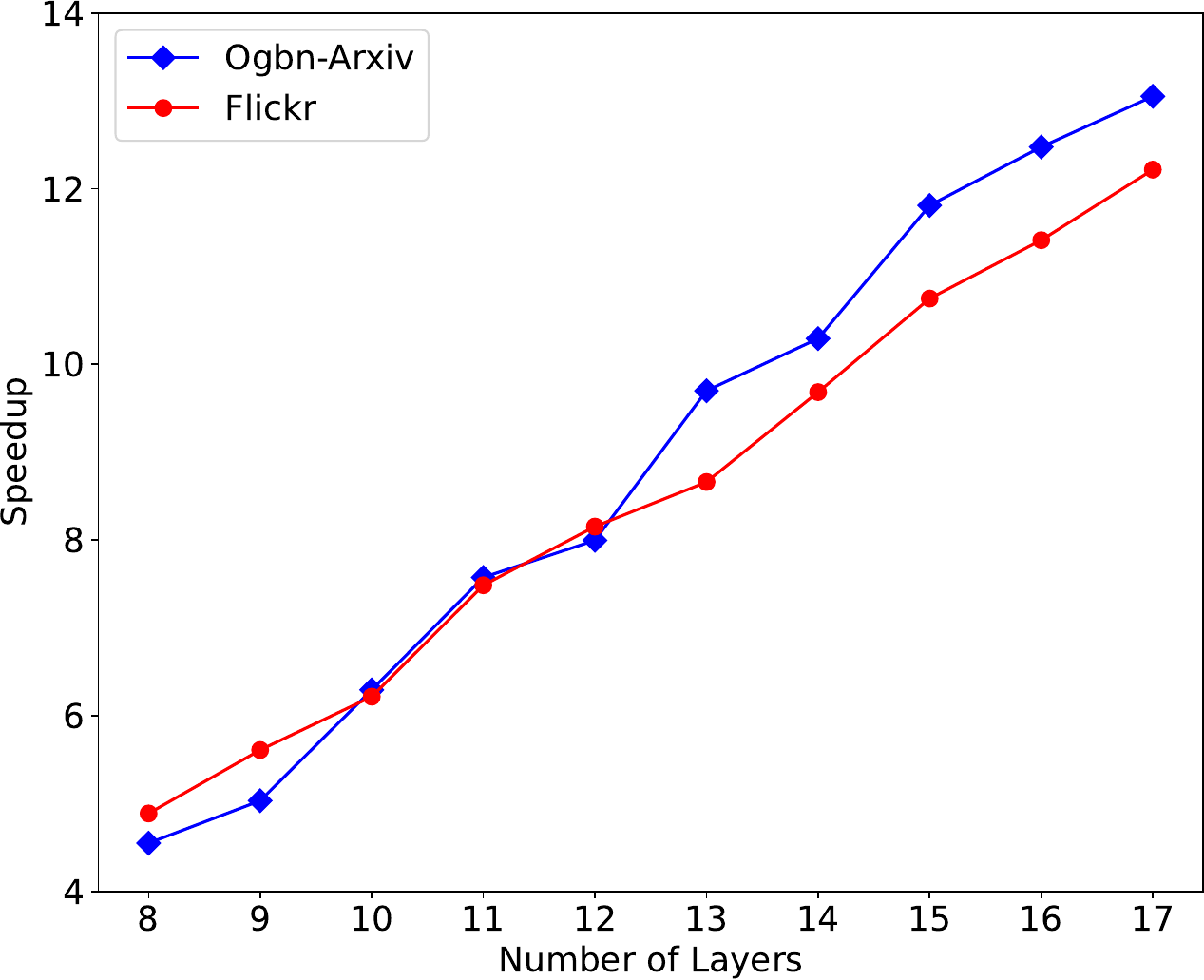}}
  \centerline{(b). The speedup}  \centerline{on large datasets.}
\end{minipage}
\caption{The speedup of the proposed pdADMM-G on different datasets concerning the number of layers: the speedup increases linearly with the number of layers, and the slopes of speedups are higher on large datasets than those on small datasets.}
\label{fig:layer_speedup}
\vspace{-0.5cm}
\end{figure}
\begin{figure}
    \centering
    \includegraphics[width=\linewidth]{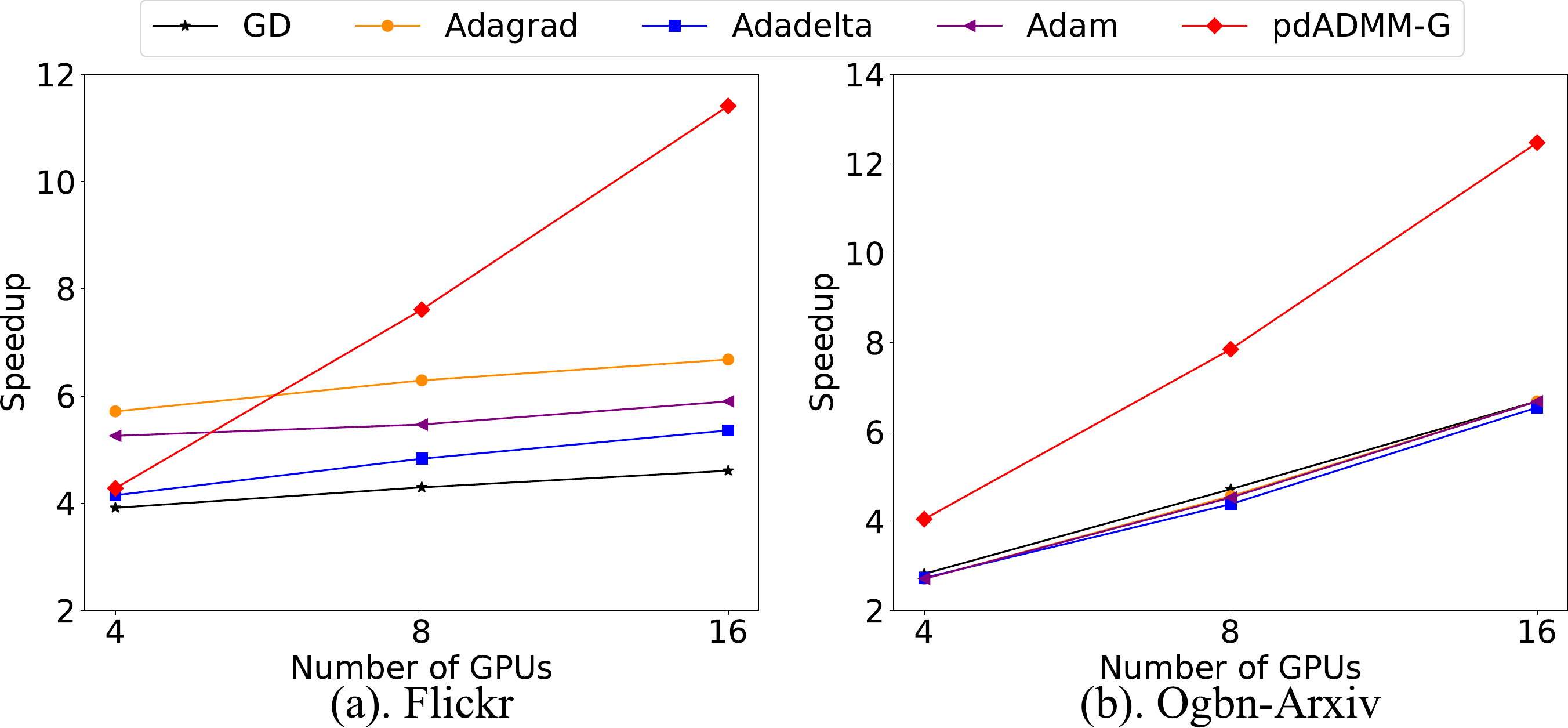}
    \caption{The speedup of all methods on two large datasets concerning the number of GPUs:  speedups of the proposed pdADMM-G are higher than these of all comparison methods.}
    \label{fig:compare_speedup}
    \vspace{-0.6cm}
  \end{figure}
\subsection{Speedup}
\indent Next, we investigate the speedup of the pdADMM-G algorithm in the large deep GA-MLP models. The running time per epoch was an average of $10$ epochs. $\rho$ and $\nu$ were both set to $10^{-3}$. We investigate the speedup concerning two factors: the number of layers and the number of GPUs.\\
 \indent For the relationship between the speedup and the number of layers, the pdADMM-G algorithm in the GA-MLP models with $\num{4000}$ neurons was tested. The number of layers ranged from $8$ to $17$. The speedup on small datasets and large datasets are shown in Figure \ref{fig:layer_speedup}(a) and Figure \ref{fig:layer_speedup}(b), respectively. Overall, the speedup of the proposed pdADMM-G increases linearly with the number of layers. For example, the speedups on the Cora dataset and the Amazon Computers dataset rise from $3$ and $3.5$ gradually to $4$ and $4.5$, respectively. The speedup on the PubMed dataset achieves the lowest with a value of less than $3$, whereas that on the Coauthor CS dataset at least doubles that on any other small dataset, with a peak of $6$. Moreover, the speedup is more obvious on large datasets. For example, when the slopes of speedups are compared, the slope on the Flickr dataset is at least five times much steeper than that on the Coauthor CS dataset. The same trend is applied to the Ogbn-Arxiv dataset. This means that our proposed pdADMM-G algorithm is more suitable for large datasets.\\
  \indent For the relationship between the speedup and the number of GPUs, we set up a large GA-MLP model with $16$ layers and $\num{4000}$ neurons and kept all hyperparameters in the previous experiment. The speedup of our proposed pdADMM-G algorithm was compared with all comparison methods. Figure \ref{fig:compare_speedup} shows all speedups on two large datasets. The proposed pdADMM-G algorithm achieves a higher speedup than any GD-based method. For example, the speedups of 8 GPUs are nearly $8$ on the Flickr dataset and the Ogbn-Arxiv dataset, while the best speedups achieved via comparison methods are in the vicinity of $6$ and $5$ on two datasets, respectively. We also observe that while speedups of all methods scale linearly with the number of GPUs, the slopes of our proposed pdADMM-G algorithm are steeper than these of any comparison method. For example, the slope of our proposed pdADMM-G algorithm on the Flickr dataset is more than 10 times steeper than that of Adam. All comparison methods show similar flat slopes, and they achieve a higher slope of the speedup on the Ogbn-Arxiv dataset than that on the Flickr dataset.\\
  \indent In summary, the speedup of our proposed pdADMM-G algorithm scales linearly with the number of layers and the number of GPUs. Moreover, its speedup is superior to any other comparison method significantly by more than $10$ times.
  \begin{figure*}
    \centering
    \includegraphics[width=\linewidth]{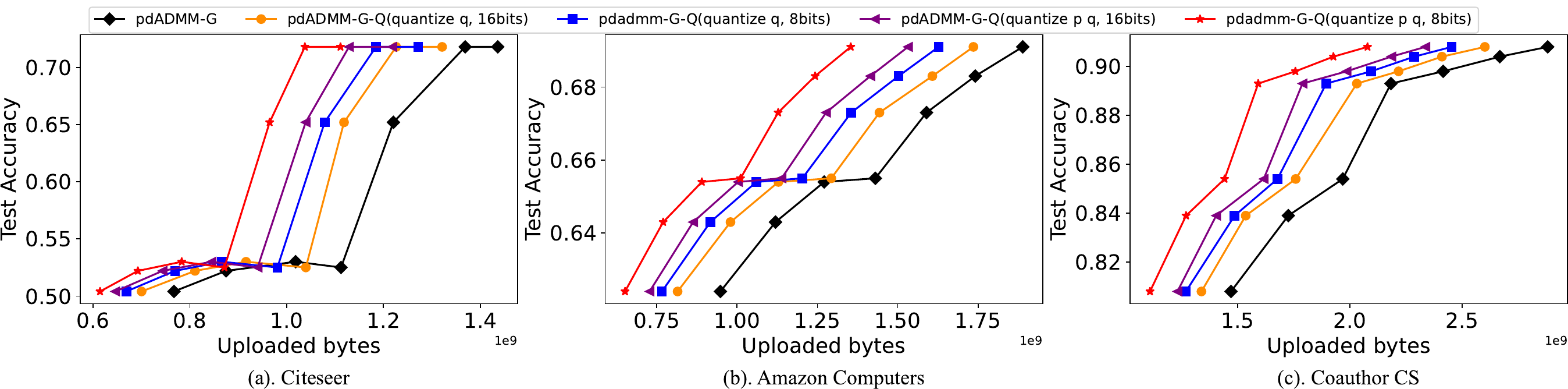}
    \caption{Communication overheads of two proposed algorithms on three datasets: the quantization concerning both $\textbf{p}$ and $\textbf{q}$ using $8$ bits reduces the communication overheads by up to $45\%$ without loss of performance.}
    \label{fig:overhead}
    \vspace{-0.6cm}
\end{figure*}
\vspace{-0.5cm}
 \subsection{Communication Overheads}
Then, it is necessary to explore how many communication overheads can be reduced using the proposed quantization technique on different quantization levels. To achieve this, we established a large GA-MLP model with $10$ layers, each of which consists of $\num{1000}$ neurons. We set up three quantization cases: no quantization, the quantization concerning $\textbf{p}$ only, and the quantization concerning both $\textbf{p}$ and $\textbf{q}$. For every quantization case, we also set up two different quantization sizes: $8$ bits and $16$ bits. Figure \ref{fig:overhead} demonstrates the relationship between the test accuracy and communication overheads for different quantization cases and sizes on three datasets. Overall, communication overheads can be reduced significantly by the proposed quantization technique. The amount of reduction depends on different quantization cases and sizes. Generally speaking, the more variables are quantized and the fewer bits are compressed, then the more savings in communications can be achieved. Take the Citeseer dataset as an example,  while all algorithms reach the same test accuracy above $70\%$, the proposed pdADMM-G (i.e. no quantization) consumes the most communication costs with the value of around $1.4\times10^9$ bytes. If the variable $\textbf{p}$ is quantized using $16$ bits, the communication overhead drops by $10\%$, and then using $8$ bits saves another $5\%$. When variables $\textbf{p}$ and $\textbf{q}$ are both quantized, the communication overhead tumbles down to
$1.2\times10^9$ bytes, which means decreases by $16.7\%$ when it is compared with the case where only $\textbf{p}$ is quantized. When variables are compressed to $8$ bits instead of $16$ bits, the communication overhead slips further to $1.1\times10^9$, a nearly $30\%$ decline. The same trend is applied to the other two datasets, and they accomplish a shrink of communication overheads by $33\%$ and $45\%$, respectively. This demonstrates that our proposed quantization technique is effective for reducing unnecessary communication costs without loss of performance. We also observe that the Coauthor CS dataset is the largest dataset among the three, and it accomplishes the greatest communication reduction. 
\begin{table}[]
    \centering
    \begin{tabular}{c|c|c|c}
    \hline\hline
    Dataset&Cora&PubMed&Citeseer\\
    \hline
         GD&0.730$\pm$0.022&0.638$\pm$ 0.080&0.637$\pm$0.040

  \\\hline
         Adadelta&0.671$\pm$0.064&0.705$\pm$0.038&0.620$\pm$0.016
 \\\hline
         Adagrad&0.726$\pm$0.025&0.753$\pm$	0.015&0.601$\pm$0.037\\\hline
         Adam&0.725$\pm$0.036&0.742$\pm$0.007&0.631$\pm$0.018
\\\hline
         pdADMM-G&0.784$\pm$0.003&\textbf{0.784$\pm$0.004}&0.709$\pm$0.003
\\\hline
         pdADMM-G-Q&\textbf{0.788$\pm$0.003}&0.782$\pm$0.003&\textbf{0.712$\pm$ 0.001}
\\\hline\hline
Dataset&\tabincell{c}{Amazon\\ Computers}&\tabincell{c}{Amazon \\Photo}&\tabincell{c}{Coauthor\\ CS}\\
\hline
GD&0.646$\pm$0.032&0.735$\pm$	0.169&0.884$\pm$0.010\\
\hline
Adadelta&0.136$\pm$0.060&0.369$\pm$	0.045&0.787$\pm$0.086\\
\hline
Adagrad&0.688$\pm$0.023&0.813$\pm$	0.018&0.887$\pm$0.007\\
\hline
Adam&0.724$\pm$0.010&0.855$\pm$0.009&0.883$\pm$0.009\\
\hline
pdADMM-G&\textbf{0.735$\pm$0.006}&\textbf{0.856$\pm$0.011}&\textbf{0.915$\pm$0.004}\\
\hline
pdADMM-G-Q&0.687$\pm$	0.054&0.832$\pm$0.010&0.914$\pm$0.003\\
\hline\hline
 Dataset&\tabincell{c}{Coauthor \\Physics}&Flickr& Ogbn-Arxiv\\
     \hline
         GD&0.909$\pm$0.007&0.466$\pm$0.007&0.361$\pm$0.063  \\\hline
         Adadelta&0.915$\pm$0.014&0.461$\pm$0.008&0.523$\pm$	0.030 \\\hline
         Adagrad&0.916$\pm$0.012&0.481$\pm$0.003&0.567$\pm$0.016\\\hline
         Adam&0.912$\pm$0.016&0.512	$\pm$0.004&\textbf{0.674$\pm$	0.006}\\\hline
         pdADMM-G&\textbf{0.921$\pm$0.003}&\textbf{0.513$\pm$0.002}&0.647$\pm$0.002
         \\\hline
         pdADMM-G-Q&0.919$\pm$0.002&0.507$\pm$0.003&0.655$\pm$0.002
         \\\hline\hline
    \end{tabular}
    \caption{The test performance of all methods when the number of neurons is 100: two proposed algorithms outperform all comparison methods.}
    \vspace{-0.5cm}
    \label{tab:performance_100_neurons}
\end{table}
\begin{table}[]
    \centering
    \begin{tabular}{c|c|c|c}
    \hline\hline
    Dataset&Cora&PubMed&Citeseer\\
    \hline
         GD&0.757$\pm$0.024&0.699$\pm$0.655&0.680$\pm$0.014

  \\\hline
         Adadelta&0.717$\pm$0.063&0.722$\pm$0.696&0.564$\pm$0.028
 \\\hline
         Adagrad&0.776$\pm$0.013&0.759$\pm$0.761&0.650	$\pm$0.038\\\hline
         Adam&0.771$\pm$0.020&0.778$\pm$0.767&0.662	$\pm$0.021\\\hline
         pdADMM-G&\textbf{0.786$\pm$0.005}&0.786$\pm$0.786&\textbf{0.713$\pm$0.007}
\\\hline
         pdADMM-G-Q&\textbf{0.786$\pm$0.005}&\textbf{0.788$\pm$0.787}&0.712$\pm$0.005
\\\hline\hline
Dataset&\tabincell{c}{Amazon\\ Computers}&\tabincell{c}{Amazon \\Photo}&\tabincell{c}{Coauthor \\CS}\\
\hline
GD&0.707$\pm$0.012&0.817$\pm$0.005&0.906$\pm$0.005\\
\hline
Adadelta&0.243$\pm$0.063&0.380$\pm$0.069&0.880$\pm$0.011\\\hline
Adagrad&\textbf{0.753$\pm$0.009}&0.866$\pm$0.007&0.911$\pm$0.003\\\hline
Adam&0.739$\pm$0.022&\textbf{0.880$\pm$	0.016}&0.898$\pm$0.013\\\hline
pdADMM-G&0.751$\pm$0.008&0.873$\pm$0.004&\textbf{0.920$\pm$0.002}\\\hline
pdADMM-G-Q&0.748$\pm$0.004&0.865$\pm$0.007&0.919$\pm$0.003\\
\hline\hline
 Dataset&\tabincell{c}{Coauthor \\Physics}&Flickr& Ogbn-Arxiv\\
     \hline
         GD&0.917$\pm$0.004&0.466$\pm$0.001&0.436$\pm$0.042  \\\hline
         Adadelta&0.917$\pm$0.004&0.462$\pm$0.001&0.584$\pm$0.031\\\hline
         Adagrad&0.914$\pm$0.004&0.487$\pm$0.005&0.630$\pm$0.016\\\hline
         Adam&0.914$\pm$0.002&\textbf{0.517$\pm$0.002}&\textbf{0.682$\pm$0.010}
         \\\hline
         pdADMM-G&\textbf{0.918$\pm$0.003}&0.515$\pm$0.002& 0.655$\pm$0.001
         \\\hline
         pdADMM-G-Q&\textbf{0.918$\pm$0.002}&0.512$\pm$0.003&0.657$\pm$0.002
         \\\hline\hline
    \end{tabular}
    \caption{The test performance of all methods when the number of neurons is 500: two proposed algorithms outperform all comparison methods.}
    \vspace{-0.6cm}
    \label{tab:performance_500_neurons}
\end{table}
\vspace{-0.3cm}
 \subsection{Performance}
\indent Finally, we evaluate the performance of two proposed algorithms against all comparison methods on nine benchmark datasets. We set up two standard GA-MLP models with $10$ layers but different neurons: the first model has $100$ neurons, while the second model has $500$ neurons. Following the greedy layerwise training strategy \cite{bengio2006greedy}, we firstly trained a two-layer GA-MLP model, and then three more layers were added to training, and finally, all $10$ layers were involved. The number of epochs was set to $200$. We repeated all experiments five times and reported their means and the standard deviations. Due to space limit, hyperparameter settings and the performance of validation sets are shown in Section \ref{sec:hyperparameter} and \ref{sec:validation performance} in the Appendix, respectively.\\
\indent Table \ref{tab:performance_100_neurons}  demonstrates the performance of all methods when the number of neurons is $\num{100}$. In summary, the two proposed algorithms outperform all comparison methods slightly: they occupy the best algorithms on eight datasets out of the total nine datasets. For example, they both achieve $78\%$ test accuracy on the Cora dataset, whereas the best comparison method is GD, which only reaches $73\%$ test accuracy, and is at least $6\%$ lower than the two proposed algorithms. As another example, two proposed algorithms accomplish $78\%$ test accuracy on the PubMed dataset, $4\%$ better than that achieved by Adagrad, whose performance is the best aside from the two proposed algorithms. The Citeseer dataset shows the largest performance gap between the two proposed algorithms and all comparison methods. Two proposed algorithms reach the level of $70\%$ test accuracy, whereas all comparison methods fall in the vicinity of  $60\%$ test accuracy. In other words, the two proposed algorithms outperform all comparison methods by more than $10\%$. For two proposed algorithms, the proposed pdADMM-G algorithm outperforms marginally the proposed pdADMM-G-Q algorithm due to the quantization technique. Their largest performance gap is $5\%$, which is achieved on the Amazon Computers dataset. The Adam is the best comparison method overall, and it serves as the best algorithm on the Ogbn-Arxiv dataset. The Adadelta performs the worst among all comparison methods, whose performance is significantly lower than any other method on several datasets such as the Amazon Computers dataset, the Amazon Photo dataset, and the Coauthor CS dataset. Last but not least, the standard deviations of all methods remain low, and this shows that they are robust to different initializations.\\
\indent Table \ref{tab:performance_500_neurons} shows the performance of all methods when the number of neurons is $500$. In general,  two proposed algorithms still reach a better performance than all comparison methods, but the gap is more narrow. For example, in Table \ref{tab:performance_100_neurons}, the proposed pdADMM-G algorithm achieves the best on the Amazon Computers dataset. However, it is surpassed by Adagrad slightly in Table \ref{tab:performance_500_neurons}. We also observe that a GA-MLP model with $\num{500}$ neurons performs better than that with $\num{100}$ neurons, which are trained by the same algorithm. This makes sense since the wider a model is, the more expressiveness it is equipped with.
\vspace{-0.3cm}
\section{Conclusion}
\label{sec:conclusion}
 The GA-MLP models are attractive to the deep learning community due to potential resistance to some problems of GNNs such as over-smoothing and over-squashing. In this paper, we propose a novel pdADMM-G algorithm to achieve parallel training of GA-MLP models, which is accomplished by breaking the layer dependency. The extended pdADMM-G-Q algorithm reduces communication overheads by the introduction of the quantization technique. Their theoretical convergence to a (quantized) stationary point of the problem is guaranteed with a sublinear convergence rate $o(1/k)$, where $k$ is the number of iterations. Extensive experiments verify that the two proposed algorithms not only converge in terms of objectives and residuals, and accelerate the training of deep GA-MLP models, but also stand out among all the existing state-of-the-art optimizers on nine benchmark datasets. Moreover, the pdADMM-G-Q algorithm reduces communication overheads by up to $45\%$ without loss of performance.
 \vspace{-0.3cm}
\bibliographystyle{ieeetr}
\bibliography{Reference}
\newpage
\onecolumn
\appendix
\subsection{Solutions to Subproblems of the pdADMM-G Algorithm}
\label{sec:pdADMM-G subproblem}
\indent We discuss how to solve all subproblems generated by pdADMM-G in detail.\\
\subsubsection{Update $\textbf{p}^{k+1}$}
\indent The variable $\textbf{p}^{k+1}$ is updated as follows:
\begin{align*}
    p^{k+1}_l&\leftarrow \arg\min\nolimits_{p_l} L_\rho(\textbf{p},\textbf{W}^k,\textbf{b}^k,\textbf{z}^k,\textbf{q}^k,\textbf{u}^k)=\phi(p_l,W^k_l,b^k_l,z^k_l,q^k_{l-1},u^k_{l-1}).
\end{align*}
Because $W_l$ and $p_l$ are coupled in $\phi$, solving $p_l$ should require the time-consuming operation of matrix inversion of $W_l$. To handle this, we apply similar quadratic approximation techniques as used in dlADMM \cite{wang2019admm} as follows:
\begin{align}
p_l^{k+1}&\leftarrow\arg\min_{p_l} U_l(p_l;\tau_l^{k+1}),
    \label{eq:update p}
\end{align}
where $
    U_l(p_l;\tau^{k+1}_l)=\phi(p^k_l,W^k_l,b^k_l,z^k_l,q^k_{l-1},u^k_{l-1})+\nabla_{p^k_l}\phi^T(p^k_l,W^k_l,b^k_l,z^k_l,q^k_{l-1},u^k_{l-1})(p_l-p^k_l)+({\tau}^{k+1}_l/2)\Vert p_{l}-p^k_{l}\Vert^2_2
$, and $\tau^{k+1}_l>0$ is a parameter. $\tau_l^{k+1}$ should satisfy $\phi(p^{k+1}_l,W^k_l,b^k_l,z^k_l,q^k_{l-1},u^k_{l-1})\leq U_l(p^{k+1}_l;\tau^{k+1}_l)$.  The solution to Equation \eqref{eq:update p} is:
${p}^{k+1}_{l}\leftarrow p^k_{l}-\nabla_{p^k_{l}}\phi(p^k_l,W^k_l,b^k_l,z^k_l,q^k_{l-1},u^k_{l-1})/{\tau}^{k+1}_{l}$.\\
\subsubsection{Update $\textbf{W}^{k+1}$}
\indent The variable $\textbf{W}^{k+1}$ is updated as follows:
\begin{align*}
    W^{k+1}_l&\leftarrow \arg\min\nolimits_{W_l} L_\rho(\textbf{p}^{k+1},\textbf{W},\textbf{b}^k,\textbf{z}^{k},\textbf{q}^{k},\textbf{u}^{k}) =\arg\min\nolimits_{W_l} \begin{cases} \phi(p^{k+1}_1,W_1,b^k_1,z^k_1),
    &l=1,\\\phi(p^{k\!+\!1}_l,W_l,b^k_l,z^k_l,q^k_{l\!-\!1},u^k_{l\!-\!1}),&1<\!l\!\leq\! L.
    \end{cases}
\end{align*}
Similar to updating $p_l$, the following subproblem should be solved instead:
\begin{align}
    W_l^{k+1}&\leftarrow\arg\min\nolimits_{W_l} V_l(W_l;\theta_l^{k+1}), \label{eq:update W}
\end{align}
where 
\begin{align*}
&V_1(W_1;\theta^{k+1}_1)= \phi(p^{k+1}_1,W^k_1,b^k_1,z^k_1)+\!\nabla_{W^k_1}\phi^T(p^{k+1}_1,W^k_1,b^k_1,z^k_1)(W_1\!-\!W^k_1)   +({\theta}^{k+1}_l/2)\Vert W_{1}-W^k_{1}\Vert^2_2,\\&
 V_l(W_l;\theta^{k+1}_l)=
   \phi(p^{k+1}_l,W^k_l,b^k_l,z^k_l,q^k_{l-1},u^k_{l-1})+\nabla_{W^k_l}\phi^T(p^{k+1}_l,W^k_l,b^k_l,z^k_l,q^k_{l-1},u^k_{l-1})(W_l-W^k_l)+({\theta}^{k+1}_l/2)\Vert W_{l}-W^k_{l}\Vert^2_2,
\end{align*}
and $\theta^{k+1}_l$ is a parameter, which should satisfy $\phi(p^{k+1}_1,W^{k+1}_1,b^k_1,z^k_1)\leq V(W^{k+1}_1;\theta^{k+1}_1)$ and $\phi(p^{k+1}_l,W^{k+1}_l,b^k_l,z^k_l,q^k_{l-1},u^k_{l-1})\leq V(W^{k+1}_l;\theta^{k+1}_l)(1<l<L)$. The solution to Equation \eqref{eq:update W} is shown as follows:
\begin{align*}
    {W}^{k\!+\!1}_{l}\!\leftarrow\! W^k_{l}\!-\!\begin{cases}\nabla_{W^k_1}\phi(p^{k+1}_1\!,\!W^k_1\!,\!b^k_1\!,\!z^k_1)/\theta^{k+1}_l,& l=1,\\\nabla_{W^k_l}\phi(p^{k\!+\!1}_l\!,\!W^k_l\!,\!b^k_l\!,\!z^k_l\!,\!q^k_{l\!-\!1}\!,\!u^k_{l\!-\!1})/\theta^{k\!+\!1}_l,& 1\!<l\!\leq\! L.
    \end{cases}
\end{align*}
\subsubsection{Update $\textbf{b}^{k+1}$}
\indent The variable $\textbf{b}^{k+1}$ is updated as follows:
\begin{align*}
    b^{k+1}_l&\leftarrow \arg\min\nolimits_{b_l} L_\rho(\textbf{p}^{k+1},\textbf{W}^{k+1},\textbf{b},\textbf{z}^{k},\textbf{q}^{k},\textbf{u}^{k}) =\arg\min\nolimits_{b_l} \begin{cases} \phi(p^{k+1}_1,W^{k+1}_1,b_1,z^k_1),
    &l\!=\!1,\\\phi(p^{k+1}_l,W^{k\!+\!1}_l,b_l,z^k_l,q^k_{l\!-\!1},u^k_{l\!-\!1}),&1<\!l\!\leq L.
    \end{cases}
\end{align*}
Similarly, we solve the following subproblems instead:
\begin{align*}
&b^{k+1}_1\leftarrow \arg\min\nolimits_{b_1} \phi(p^{k+1}_1,W^{k+1}_1,b^k_1,z^k_1)+\nabla_{b^k_1}\phi^T(p^{k+1}_1,W^{k+1}_1,b^k_1,z^k_1)(b_l-b^k_l)+(\nu/2)\Vert b_{l}-b^k_{l}\Vert^2_2,
\end{align*}
\begin{align}
&\nonumber b^{k+1}_l\leftarrow \arg\min\nolimits_{b_l} \phi(p^{k+1}_l,W^{k+1}_l,b^k_l,z^k_l,q^k_{l-1},u^k_{l-1})+\nabla_{b^k_l}\phi^T(p^{k+1}_l,W^{k+1}_l,b^k_l,z^k_l,q^k_{l-1},u^k_{l-1})(b_l-b^k_l)\\&+(\nu/2)\Vert b_{l}-b^k_{l}\Vert^2_2 (1< l\leq L). \label{eq:update b}
\end{align}
\indent The solution to Equation \eqref{eq:update b} is:
\begin{align*}
    {b}^{k\!+\!1}_{l}\leftarrow b^k_{l}\!-\!\begin{cases}\nabla_{b^k_1}\phi(p^{k\!+\!1}_1,W^{k\!+\!1}_1,b^k_1,z^k_1)/\nu,& l=1,\\\nabla_{b^k_l}\phi(p^{k\!+\!1}_l,W^{k\!+\!1}_l,b^k_l,z^k_l,q^k_{l\!-\!1},u^k_{l\!-\!1})/\nu,& 1\!<\!l\!\leq L.
    \end{cases}
\end{align*}
\subsubsection{Update $\textbf{z}^{k+1}$}
\indent The variable
$\textbf{z}^{k+1}$ is updated as follows:
\begin{align}
    & z^{k\!+\!1}_l\!\leftarrow\! \arg\min\nolimits_{z_l} (\nu/2)\Vert z_l\!-\!W^{k\!+\!1}_l\!p^{k\!+\!1}_l\!-\!b^{k\!+\!1}_l\Vert^2_2+\!(\nu/2)\Vert q^k_l\!-
    \!f_l(z_l)\Vert^2_2\!+\!(\nu/2)\Vert z_l\!-\!z^k_l\Vert^2_2(l\!<\! L),\label{eq:update z}\\
    &z^{k\!+\!1}_L\!\leftarrow\! \arg\min\nolimits_{z_l}\! R(z_L;y)\!+\!(\nu/2)\Vert z_L\!-\!W^{k\!+\!1}_Lp^{k\!+\!1}_L\!-\!b^{k\!+\!1}_L\Vert^2_2.\label{eq:update zl}
\end{align}
where a quadratic term $(\nu/2)\Vert z_l-z^k_l\Vert^2_2$ is added in Equation \eqref{eq:update z} to control $z^{k+1}_l$ to close to $z^{k}_l$.
Equation \eqref{eq:update zl} is convex, which can be solved by Fast Iterative Soft Thresholding Algorithm (FISTA) \cite{beck2009fast}.\\
For Equation \eqref{eq:update z}, nonsmooth activations usually lead to  closed-form solutions \cite{wang2019admm,wang2017nonconvex}. For example, for ReLU $f_l(z_l)=\max(z_l,0)$, the solution to Equation \eqref{eq:update z} is shown as follows:
\begin{align*}
    z^{k+1}_l=
    \begin{cases} \min((W^{k\!+\!1}_{l}p^{k\!+\!1}_{l}\!+\!b^{k+1}_l\!+\!z^k_l)/2,0),&z^{k\!+\!1}_l\leq0,\\
    \max((W^{k\!+\!1}_{l}p^{k\!+\!1}_{l}\!+\!b^{k\!+\!1}_l\!+\!q^k_l\!+\!z^k_l)/3,0),&z^{k+1}_l\geq0.\\
    \end{cases}
\end{align*}
For smooth activations such as tanh and sigmoid, a lookup-table is recommended \cite{wang2019admm}.\\
\subsubsection{Update $\textbf{q}^{k+1}$}
\indent The variable
$\textbf{q}^{k+1}$ is updated as follows:
\begin{align} q^{k+1}_l&\leftarrow \arg\min\nolimits_{q_l} L_\rho(\textbf{p}^{k+1},\textbf{W}^{k+1},\textbf{b}^{k+1},\textbf{z}^{k+1},\textbf{q},\textbf{u}^k)=\arg\min\nolimits_{q_l}\phi(p^{k+1}_{l+1},W^{k+1}_{l+1},b^{k+1}_{l+1},z^{k+1}_{l+1},q_l,u^{k}_l). \label{eq:update q}
\end{align}
Equation \eqref{eq:update q} has a closed-form solution as follows:\\
\begin{align*}
    q^{k+1}_l\leftarrow (\rho p_{l+1}^{k+1}+u^k_l+\nu f_l(z^{k+1}_l))/(\rho+\nu).
\end{align*}
\subsubsection{Update $\textbf{u}^{k+1}$}
\indent The variable
$\textbf{u}^{k+1}$ is updated as follows:
\begin{align}
u^{k+1}_l&\leftarrow u^k_l+\rho(p^{k+1}_{l+1}-q^{k+1}_{l}). \label{eq:update u}
\end{align}
\subsection{Solutions to Subproblems of the pdADMM-G-Q Algorithm}
\label{sec:pdADMM-G-Q subproblem}
\indent Obviously, the only difference between the pdADMM-G-Q algorithm and the pdADMM-G algorithm is the $p_l$-subproblem, which is outlined in the following:\\
\begin{align}
p_l^{k+1}&\leftarrow\arg\min\nolimits_{p_l} U_l(p_l;\tau_l^{k+1})+\mathbb{I}(p_l),
    \label{eq:update p quantized}
\end{align}
where $U_l$ follows Equation \eqref{eq:update p}. The solution to Equation \eqref{eq:update p quantized} is \cite{huang2021alternating}:
${p}^{k+1}_{l}\leftarrow\arg\min\nolimits_{\delta\in \Delta}\Vert \delta-(p^k_{l}-\nabla_{p^k_{l}}\phi(p^k_l,W^k_l,b^k_l,z^k_l,q^k_{l-1},u^k_{l-1})/{\tau}^{k+1}_{l})\Vert$.\\
\indent For the pdADMM-G-Q algorithm, the variable $\textbf{p}$ is only required to be quantized (i.e. $p_l\in \Delta$) when the $p_l$-subproblem is solved (i.e. Equation \eqref{eq:update p quantized}), and the variable  $\textbf{q}$ can be any real number when it is updated (i.e. Equation \eqref{eq:update q}). However,  $\textbf{q}$ is guaranteed to fit into $\Delta$ by the linear constraint $p_{l+1}=q_l$. This design is convenient for the convergence analysis, which is detailed in the next section. One variant of the pdADMM-G-Q algorithm is to quantize $\textbf{p}$ and $\textbf{q}$ (i.e. $p_l,q_l\in \Delta$) when they are updated. In this case, the solution to Equation \eqref{eq:update q} is ${q}^{k+1}_{l}\leftarrow\arg\min\nolimits_{\delta\in \Delta}\Vert \delta-(\rho p_{l+1}^{k+1}+u^k_l+\nu f_l(z^{k+1}_l))/(\rho+\nu)\Vert$.
\subsection{Convergence Proofs}
\label{sec:convergence proof}
\subsubsection{Preliminary Results}
\begin{lemma}
\label{lemma: q_opt}
It holds for every $k\in \mathbb{N}$ and $l=1,\cdots,L-1$ that
\begin{align*} 
u^k_l=\nu(q^{k}_l-f_l(z^{k}_l)).
\end{align*}
\end{lemma}
\begin{proof}
This follows directly from the optimality condition of $q^k_l$ and Equation \eqref{eq:update u}.
\end{proof}
\begin{lemma}
\label{lemma:u bound}
It holds for every $k\in \mathbb{N}$ and $l=1,\cdots,L-1$ that
\begin{align*}
    \Vert u^{k+1}_l-u^k_l\Vert\leq \nu\Vert q^{k+1}_l-q^{k}_l\Vert+\nu S\Vert z^{k+1}_l-z^{k}_l\Vert.
\end{align*}
\begin{proof}
\begin{align*}
    &\Vert u^{k+1}_l-u^k_l\Vert\\&=\Vert \nu(q^{k+1}_l-f_l(z^{k+1}_l))-\nu(q^{k}_l-f_l(z^{k}_l))\Vert \text{(Lemma \ref{lemma: q_opt})}\\&\leq\nu\Vert q^{k+1}_l-q^{k}_l\Vert+\nu\Vert f_l(z^{k+1}_l)-f_l(z^{k}_l)\Vert \text{(triangle inequality)}\\&\leq \nu\Vert q^{k+1}_l-q^{k}_l\Vert+\nu S\Vert z^{k+1}_l-z^{k}_l\Vert \text{(Assumption \ref{ass:lipschitz continuous})}.
\end{align*}
\end{proof}
\end{lemma}
\begin{lemma}
\label{lemma:u square bound}
It holds for every $k\in \mathbb{N}$ and $l=1,\cdots,L-1$ that
\begin{align*}
    \Vert u^{k+1}_l-u^k_l\Vert^2_2\leq 2\nu^2(\Vert q^{k+1}_l-q^{k}_l\Vert^2_2+S^2\Vert z^{k+1}_l-z^{k}_l\Vert^2_2).
\end{align*}
\begin{proof}
\begin{align*}
        \Vert u^{k+1}_l-u^k_l\Vert^2_2&=\nu^2\Vert q^{k+1}_l-f_l(z^{k+1}_l)-q^{k}_l+f_l(z^{k}_l)\Vert^2_2 \text{(Lemma \ref{lemma: q_opt})}\\&\leq 2\nu^2(\Vert q^{k+1}_l-q^k_l\Vert^2_2+\Vert f_l(z^{k+1}_l)-f_l(z^{k}_l)\Vert^2_2)\text{(mean inequality)}\\&\leq2\nu^2(\Vert q^{k+1}_l-q^k_l\Vert^2_2+S^2\Vert z^{k+1}_l-z^{k}_l\Vert^2_2)\text{(Assumption \ref{ass:lipschitz continuous})}.
\end{align*}

\end{proof}
\end{lemma}
\begin{lemma}
\label{lemma:optimality 1}
For every $k\in \mathbb{N}$, it holds that
\begin{align}  &
  L_\rho(\textbf{p}^{k},\textbf{W}^{k},\textbf{b}^k,\textbf{z}^{k},\textbf{q}^k,\textbf{u}^{k})\!-\! L_\rho(\textbf{p}^{k+1},\textbf{W}^{k},\textbf{b}^k,\textbf{z}^{k},\textbf{q}^k,\textbf{u}^{k}) \geq  \sum\nolimits_{l=2}^{L}(\tau^{k+1}_l/2)\Vert  p^{k+1}_l-p^k_l\Vert^2_2, \label{eq: pdADMM-G p optimality} \\&
    L_\rho(\textbf{p}^{k\!+\!1},\textbf{W}^{k},\textbf{b}^k,\textbf{z}^{k},\textbf{q}^k,\textbf{u}^{k})\!-\! L_\rho(\textbf{p}^{k\!+\!1},\textbf{W}^{k\!+\!1},\textbf{b}^{k},\textbf{z}^{k},\textbf{q}^k,\textbf{u}^{k})\! \geq\!  \sum\nolimits_{l\!=\!1}^{L}(\theta^{k\!+\!1}_l/2)\Vert  W^{k\!+\!1}_l\!-\!W^k_l\Vert^2_2, \label{eq: pdADMM-G W optimality}
  \\&
  L_\rho(\textbf{p}^{k\!+\!1},\textbf{W}^{k\!+\!1},\textbf{b}^k,\textbf{z}^{k},\textbf{q}^k,\textbf{u}^{k})\!-\! L_\rho(\textbf{p}^{k\!+\!1},\textbf{W}^{k\!+\!1},\textbf{b}^{k\!+\!1},\textbf{z}^{k},\textbf{q}^k,\textbf{u}^{k})\geq(\nu/2) \sum\nolimits_{l\!=\!1}^{L}\Vert b^{k\!+\!1}_l\!-\!b^k_l\Vert^2_2, \label{eq: pdADMM-G b optimality}\\& L_\rho(\textbf{p}^{k\!+\!1},\textbf{W}^{k\!+\!1},\textbf{b}^{k\!+\!1},\textbf{z}^{k},\textbf{q}^k,\textbf{u}^{k})\!-\! L_\rho(\textbf{p}^{k\!+\!1},\textbf{W}^{k\!+\!1},\textbf{b}^{k\!+\!1},\textbf{z}^{k\!+\!1},\textbf{q}^k,\textbf{u}^{k})\geq(\nu/2) \sum\nolimits_{l\!=\!1}^{L}\Vert z^{k\!+\!1}_l\!-\!z^k_l\Vert^2_2, \label{eq: pdADMM-G z optimality}\\
  & \beta_\rho(\textbf{p}^{k},\textbf{W}^{k},\textbf{b}^k,\textbf{z}^{k},\textbf{q}^k,\textbf{u}^{k})\geq \beta_\rho(\textbf{p}^{k+1},\textbf{W}^{k},\textbf{b}^k,\textbf{z}^{k},\textbf{q}^k,\textbf{u}^{k}), \label{eq: pdADMM-G-q p optimality} \\&
    \beta_\rho(\textbf{p}^{k\!+\!1},\textbf{W}^{k},\textbf{b}^k,\textbf{z}^{k},\textbf{q}^k,\textbf{u}^{k})\!-\! \beta_\rho(\textbf{p}^{k\!+\!1},\textbf{W}^{k\!+\!1},\textbf{b}^{k},\textbf{z}^{k},\textbf{q}^k,\textbf{u}^{k})\! \geq\!  \sum\nolimits_{l\!=\!1}^{L}(\theta^{k\!+\!1}_l/2)\Vert  W^{k\!+\!1}_l\!-\!W^k_l\Vert^2_2, \label{eq: pdADMM-G-q W optimality}
  \\&
  \beta_\rho(\textbf{p}^{k\!+\!1},\textbf{W}^{k\!+\!1},\textbf{b}^k,\textbf{z}^{k},\textbf{q}^k,\textbf{u}^{k})\!-\! \beta_\rho(\textbf{p}^{k\!+\!1},\textbf{W}^{k\!+\!1},\textbf{b}^{k\!+\!1},\textbf{z}^{k},\textbf{q}^k,\textbf{u}^{k})\geq(\nu/2) \sum\nolimits_{l\!=\!1}^{L}\Vert b^{k\!+\!1}_l\!-\!b^k_l\Vert^2_2, \label{eq: pdADMM-G-q b optimality}\\& \beta_\rho(\textbf{p}^{k\!+\!1},\textbf{W}^{k\!+\!1},\textbf{b}^{k\!+\!1},\textbf{z}^{k},\textbf{q}^k,\textbf{u}^{k})\!-\! \beta_\rho(\textbf{p}^{k\!+\!1},\textbf{W}^{k\!+\!1},\textbf{b}^{k\!+\!1},\textbf{z}^{k\!+\!1},\textbf{q}^k,\textbf{u}^{k})\geq(\nu/2) \sum\nolimits_{l\!=\!1}^{L}\Vert z^{k\!+\!1}_l\!-\!z^k_l\Vert^2_2. \label{eq: pdADMM-G-q z optimality}
\end{align}
\end{lemma}
\begin{proof}
Generally, all inequalities can be obtained by applying optimality conditions of updating $\textbf{p}$, $\textbf{W}$, $\textbf{b}$ and $\textbf{z}$, respectively. We only prove Inequalities \eqref{eq: pdADMM-G p optimality},  \eqref{eq: pdADMM-G b optimality}, \eqref{eq: pdADMM-G z optimality} and \eqref{eq: pdADMM-G-q p optimality}. This is because Inequalities \eqref{eq: pdADMM-G W optimality} and \eqref{eq: pdADMM-G-q W optimality}   follow the same routine of Inequality \eqref{eq: pdADMM-G p optimality}, Inequality \eqref{eq: pdADMM-G-q b optimality}   follows the same routine of Inequality \eqref{eq: pdADMM-G b optimality}, and Inequality \eqref{eq: pdADMM-G-q z optimality}   follows the same routine of Inequality \eqref{eq: pdADMM-G z optimality}.\\
\indent Firstly, we focus on Inequality \eqref{eq: pdADMM-G p optimality}. The choice of $\tau^{k+1}_l$ requires
\begin{align}
\phi(p^{k+1}_l,W^k_l,b^k_l,z^k_l,q^k_{l-1},u^k_{l-1})\leq U_l(p^{k+1}_l;\tau^{k+1}_l). \label{eq:stop condition}   
\end{align}
Moreover, the optimality condition of Equation \eqref{eq:update p} leads to
\begin{align}
    \nabla_{p^k_l}\phi(p^k_l,W^k_l,b^k_l,z^k_l,q^k_{l-1},u^k_{l-1})+\tau^{k+1}_l(p^{k+1}_l-p^k_l)=0. \label{eq:p optimality condition}
\end{align}
Therefore
\begin{align*}
    &L_\rho(\textbf{p}^{k},\textbf{W}^{k},\textbf{b}^k,\textbf{z}^{k},\textbf{q}^k,\textbf{u}^{k})\!-\! L_\rho(\textbf{p}^{k+1},\textbf{W}^{k},\textbf{b}^k,\textbf{z}^{k},\textbf{q}^k,\textbf{u}^{k})\\&=\sum\nolimits_{l=2}^L(\phi(p^{k}_l,W^k_l,b^k_l,z^k_l,q^k_{l-1},u^k_{l-1})-\phi(p^{k+1}_l,W^k_l,b^k_l,z^k_l,q^k_{l-1},u^k_{l-1}))\\&\geq\sum\nolimits_{l=2}^L (\phi(p^{k}_l,W^k_l,b^k_l,z^k_l,q^k_{l-1},u^k_{l-1})-U_l(p^{k+1}_l;\tau^{k+1}_l)) \text{(Inequality \eqref{eq:stop condition})}\\&=\sum\nolimits_{l=2}^L(-\nabla_{p^k_l}\phi^T(p^k_l,W^k_l,b^k_l,z^k_l,q^k_{l-1},u^k_{l-1})(p^{k+1}_l-p^k_l)-({\tau}^{k+1}_l/2)\Vert p^{k+1}_{l}-p^k_{l}\Vert^2_2)\\&=\sum\nolimits_{l=2}^L (\tau^{k+1}_l/2)\Vert p^{k+1}_l-p^k_l\Vert^2_2\text{(Equation \eqref{eq:p optimality condition})}.
\end{align*}
Next, we prove Inequality \eqref{eq: pdADMM-G b optimality}. Because $\nabla_{{b}_1}\phi({p}_1,{W}_1,{b}_1,{z}_1)$ and  $\nabla_{{b}_l}\phi({p}_l,{W}_l,{b}_l,{z}_l,{q}_l,{u}_l)$ are   Lipschitz continuous with coefficient $\nu$. According to Lemma 2.1 in \cite{beck2009fast}, we have
\begin{align}
\nonumber
    \phi({p}^{k\!+\!1}_1,{W}^{k+1}_1,{b}^{k+1}_1,{z}^{k}_1)&\leq\phi({p}^{k\!+\!1}_1,{W}^{k+1}_1,{b}^{k}_1,{z}^{k}_1)+\nabla_{{b}^{k}_1} \phi^T({p}^{k\!+\!1}_1,{W}^{k+1}_1,{b}^{k}_1,{z}^{k}_1)(b^{k+1}_1-b^k_1)\\&+(\nu/2)\Vert b_1^{k+1}-b_1^{k}\Vert^2_2, \label{eq:b1 lipschitz}
    \\\nonumber
    \phi({p}^{k\!+\!1}_l,{W}^{k+1}_l,{b}^{k+1}_l,{z}^{k}_l,{q}^{k}_{l-1},{u}^{k}_{l-1})&\leq\phi({p}^{k\!+\!1}_l,{W}^{k+1}_l,{b}^{k}_l,{z}^{k}_l,{q}^{k}_{l-1},{u}^{k}_{l-1})\\&+\nabla_{{b}^{k}_l} \phi^T({p}^{k\!+\!1}_l,{W}^{k+1}_l,{b}^{k}_l,{z}^{k}_l,{q}^{k}_{l-1},{u}^{k}_{l-1})(b^{k+1}_l-b^k_l)+(\nu/2)\Vert b_l^{k+1}-b_l^{k}\Vert^2_2.
    \label{eq:b lipschitz}
\end{align}
Moreover, the optimality condition of Equation \eqref{eq:update b} leads to
\begin{align}
    \nabla_{b^k_1}\phi(p^k_1,W^k_1,b^k_1,z^k_1)+\nu(b^{k+1}_1-b^k_1)=0, \label{eq:b1 optimality condition}\\
    \nabla_{b^k_l}\phi(p^k_l,W^k_l,b^k_l,z^k_l,q^k_{l-1},u^k_{l-1})+\nu(b^{k+1}_l-b^k_l)=0. \label{eq:b optimality condition}
\end{align}
Therefore, we have
\begin{align*}
    &L_\rho(\textbf{p}^{k\!+\!1},\textbf{W}^{k+1},\textbf{b}^k,\textbf{z}^{k},\textbf{q}^k,\textbf{u}^{k})\!-\! L_\rho(\textbf{p}^{k\!+\!1},\textbf{W}^{k\!+\!1},\textbf{b}^{k+1},\textbf{z}^{k},\textbf{q}^k,\textbf{u}^{k})\\&=\phi({p}^{k\!+\!1}_1,{W}^{k+1}_1,{b}^k_1,{z}^{k}_1)-\phi({p}^{k\!+\!1}_1,{W}^{k+1}_1,{b}^{k+1}_1,{z}^{k}_1)\\&+\sum\nolimits_{l=2}^L (\phi({p}^{k\!+\!1}_l,{W}^{k+1}_l,{b}^k_l,{z}^{k}_l,q^k_{l-1},u^k_{l-1})-\phi({p}^{k\!+\!1}_l,{W}^{k+1}_l,{b}^{k+1}_l,{z}^{k}_l,q^k_{l-1},u^k_{l-1}))\\&\geq -\nabla_{{b}^{k}_1} \phi^T({p}^{k\!+\!1}_1,{W}^{k+1}_1,{b}^{k}_1,{z}^{k}_1)(b^{k+1}_1-b^k_1)-(\nu/2)\Vert b_1^{k+1}-b_1^{k}\Vert^2_2\\&+\sum\nolimits_{l=2}^L (-\nabla_{{b}^{k}_l} \phi^T({p}^{k\!+\!1}_l,{W}^{k+1}_l,{b}^{k}_l,{z}^{k}_l,{q}^{k}_{l-1},{u}^{k}_{l-1})(b^{k+1}_l-b^k_l)-(\nu/2)\Vert b_l^{k+1}-b_l^{k}\Vert^2_2)\\&\text{(Inequalities \eqref{eq:b1 lipschitz} and \eqref{eq:b lipschitz})}\\&=(\nu/2)\sum\nolimits_{l=1}^L\Vert b^{k+1}_l-b^k_l\Vert^2_2 \text{(Equations \eqref{eq:b1 optimality condition} and \eqref{eq:b optimality condition})}.
\end{align*}
Then we prove Inequality \eqref{eq: pdADMM-G z optimality}. Because $z^{k+1}_l$ minimizes Equation \eqref{eq:update z} and Equation \eqref{eq:update zl}, we have
\begin{align}
    \nonumber &(\nu/2)\Vert z^{k+1}_l-W^{k+1}_lp^{k+1}_l-b^{k+1}_l\Vert^2_2+(\nu/2)\Vert q^k_l-f_l(z^{k+1}_l)\Vert^2_2+(\nu/2)\Vert z^{k+1}_l-z^k_l\Vert^2_2\\&\leq (\nu/2)\Vert z^{k}_l-W^{k+1}_l p^{k+1}_l-b^{k+1}_l\Vert^2_2+(\nu/2)\Vert q^k_l-f_l(z^{k}_l)\Vert^2_2, \label{eq:z minimize}
\end{align}
and
\begin{align}
    \nonumber &R(z^k_L;y)\!+(\nu/2)\!\Vert z^k_L\!-\!W^{k\!+\!1}_Lp^{k\!+\!1}_L\!-\!b^{k\!+\!1}_L\Vert^2_2-R(z^{k+1}_L;y)\!-(\nu/2)\!\Vert z^{k+1}_L\!-\!W^{k\!+\!1}_Lp^{k\!+\!1}_L\!-\!b^{k\!+\!1}_L\Vert^2_2\\\nonumber&=R(z^k_L;y)-R(z^{k+1}_L;y)+(\nu/2)\Vert z^{k}_L-z^{k+1}_L\Vert^2_2+\nu(z^{k+1}_L-W^{k+1}_Lp^{k+1}_L-b^{k+1}_L)^T(z^{k}_L-z_L^{k+1})\\\nonumber&\text{($\Vert a-b\Vert^2_2-\Vert a-c\Vert^2_2=\Vert b-c\Vert^2_2+2(c-a)^T(b-c)$ where $a=W^{k+1}_Lp^{k+1}_L+b^{k+1}_L$, $b=z^k_L$, and $c=z^{k+1}_L$)}\\\nonumber &\geq s^T(z^k_L-z^{k+1}_L)+(\nu/2)\Vert z^{k}_L-z^{k+1}_L\Vert^2_2+\nu(z^{k+1}_L-W^{k+1}_Lp^{k+1}_L-b^{k+1}_L)^T(z^{k}_L-z_L^{k+1})\\&\nonumber\text{($s\in \partial R(z^{k+1}_L;y)$ is a subgradient of $R(z^{k+1}_L;y)$)}\\&=(\nu/2)\Vert z^{k+1}_L-z^{k}_L\Vert^2_2\label{eq:zl minimize}\\\nonumber& \text{($0\in s+\nu(z^{k+1}_L-W^{k+1}_Lp^{k+1}_L-b^{k+1}_L)$ by the optimality condition of Equation \eqref{eq:update zl})}.
\end{align}
Therefore
\begin{align*}
 &L_\rho(\textbf{p}^{k\!+\!1},\textbf{W}^{k\!+\!1},\textbf{b}^{k\!+\!1},\textbf{z}^{k},\textbf{q}^k,\textbf{u}^{k})\!-\! L_\rho(\textbf{p}^{k\!+\!1},\textbf{W}^{k\!+\!1},\textbf{b}^{k\!+\!1},\textbf{z}^{k\!+\!1},\textbf{q}^k,\textbf{u}^{k})\\&=\sum\nolimits_{i=1}^{L-1}    ((\nu/2)\Vert z^{k}_l-W^{k+1}_l p^{k+1}_l-b^{k+1}_l\Vert^2_2+(\nu/2)\Vert q^k_l-f_l(z^{k}_l)\Vert^2_2\\&-(\nu/2)\Vert z^{k+1}_l-W^{k+1}_lp^{k+1}_l-b^{k+1}_l\Vert^2_2-(\nu/2)\Vert q^k_l-f_l(z^{k+1}_l)\Vert^2_2)\\&+R(z^k_L;y)\!+(\nu/2)\!\Vert z^k_L\!-\!W^{k\!+\!1}_Lp^{k\!+\!1}_L\!-\!b^{k\!+\!1}_L\Vert^2_2-R(z^{k+1}_L;y)\!-(\nu/2)\!\Vert z^{k+1}_L\!-\!W^{k\!+\!1}_Lp^{k\!+\!1}_L\!-\!b^{k\!+\!1}_L\Vert^2_2\\&\geq(\nu/2)\sum\nolimits_{l=1}^L\Vert z^{k+1}_l-z^{k}_l\Vert^2_2 \text{(Inequalities \eqref{eq:z minimize} and \eqref{eq:zl minimize})}.
\end{align*}
Finally Inequality \eqref{eq: pdADMM-G-q p optimality} follows directly the optimality condition of $\textbf{p}^{k+1}$.
\end{proof}
\begin{lemma}
For every $k\in N$, it holds that
\begin{align}
\nonumber &L_\rho(\textbf{p}^{k\!+\!1},\textbf{W}^{k\!+\!1},\textbf{b}^{k\!+\!1},\textbf{z}^{k+1},\textbf{q}^k,\textbf{u}^{k})\!-\! L_\rho(\textbf{p}^{k\!+\!1},\textbf{W}^{k\!+\!1},\textbf{b}^{k\!+\!1},\textbf{z}^{k\!+\!1},\textbf{q}^{k+1},\textbf{u}^{k+1})\\&\geq\sum\nolimits_{l=1}^{L-1}((\rho/2-2\nu^2/\rho-\nu/2)\Vert q^{k+1}_l-q^k_l\Vert^2_2-(2\nu^2 S^2/\rho)\Vert z^{k+1}_l-z^k_l\Vert^2_2)
\label{eq: pdADMM-G q optimality}\\\nonumber &\beta_\rho(\textbf{p}^{k\!+\!1},\textbf{W}^{k\!+\!1},\textbf{b}^{k\!+\!1},\textbf{z}^{k+1},\textbf{q}^k,\textbf{u}^{k})\!-\! \beta_\rho(\textbf{p}^{k\!+\!1},\textbf{W}^{k\!+\!1},\textbf{b}^{k\!+\!1},\textbf{z}^{k\!+\!1},\textbf{q}^{k+1},\textbf{u}^{k+1})\\&\geq\sum\nolimits_{l=1}^{L-1}((\rho/2-2\nu^2/\rho-\nu/2)\Vert q^{k+1}_l-q^k_l\Vert^2_2-(2\nu^2 S^2/\rho)\Vert z^{k+1}_l-z^k_l\Vert^2_2)
\label{eq: pdADMM-G-q q optimality}.
\end{align}
\end{lemma}
\begin{proof}
\indent We only prove Inequality \eqref{eq: pdADMM-G q optimality} because Inequality \eqref{eq: pdADMM-G-q q optimality} follows the same routine of Inequality \eqref{eq: pdADMM-G q optimality}.
\begin{align*}
    &L_\rho(\textbf{p}^{k\!+\!1},\textbf{W}^{k\!+\!1},\textbf{b}^{k\!+\!1},\textbf{z}^{k+1},\textbf{q}^k,\textbf{u}^{k})\!-\! L_\rho(\textbf{p}^{k\!+\!1},\textbf{W}^{k\!+\!1},\textbf{b}^{k\!+\!1},\textbf{z}^{k\!+\!1},\textbf{q}^{k+1},\textbf{u}^{k+1})\\&=\sum\nolimits_{l=1}^{L-1} ((\nu/2)\Vert f_l(z^{k+1}_l)-q^k_l\Vert^2_2-(\nu/2)\Vert f_l(z^{k+1}_l)-q^{k+1}_l\Vert^2_2-(u^{k+1}_l)^T(q^k_l-q^{k+1}_l)\\&+(\rho/2)\Vert q^{k+1}_l-q^{k}_l\Vert^2_2-(1/\rho)\Vert u^{k+1}_l-u^k_l\Vert^2_2)\\&=\sum\nolimits_{l=1}^{L-1} ((\nu/2)\Vert f_l(z^{k+1}_l)-q^k_l\Vert^2_2-(\nu/2)\Vert f_l(z^{k+1}_l)-q^{k+1}_l\Vert^2_2-\nu(q^{k+1}_l-f_l(z^{k+1}_l))^T(q^k_l-q^{k+1}_l)\\&+(\rho/2)\Vert q^{k+1}_l-q^{k}_l\Vert^2_2-(1/\rho)\Vert u^{k+1}_l-u^k_l\Vert^2_2)\text{(Lemma \ref{lemma: q_opt})}\\&\geq\sum\nolimits_{l=1}^{L-1}( -(\nu/2)\Vert q^{k+1}_l-q^k_l\Vert^2_2+(\rho/2)\Vert q^{k+1}_l-q^k_l\Vert^2_2-(1/\rho)\Vert u^{k+1}_l-u^k_l\Vert^2_2)\\&\text{($ -\nu(q_l-f_l(z^{k+1}_l))=-(\nu/2)\nabla_{q_l}\Vert q_l-f_l(z^{k+1}_l)\Vert^2_2$ is lipschitz continuous concerning $q_l$ and Lemma 2.1 in \cite{beck2009fast})}\\&\geq\sum\nolimits_{l=1}^{L-1}( -(\nu/2)\Vert q^{k+1}_l-q^k_l\Vert^2_2+(\rho/2)\Vert q^{k+1}_l-q^k_l\Vert^2_2-(2\nu^2/\rho)\Vert q^{k+1}_l-q^k_l\Vert^2_2-(2\nu^2 S^2/\rho)\Vert z^{k+1}_l-z^k_l\Vert^2_2)\\&\text{(Lemma \ref{lemma:u square bound})}\\&=\sum\nolimits_{l=1}^{L-1}((\rho/2-2\nu^2/\rho-\nu/2)\Vert q^{k+1}_l-q^k_l\Vert^2_2-(2\nu^2 S^2/\rho)\Vert z^{k+1}_l-z^k_l\Vert^2_2).\end{align*}
\end{proof}
\subsubsection{Proof of Lemma \ref{lemma:objective decrease}}
\begin{proof}
We sum up Inequalities \eqref{eq: pdADMM-G p optimality}, \eqref{eq: pdADMM-G W optimality}, \eqref{eq: pdADMM-G b optimality}, \eqref{eq: pdADMM-G z optimality}, and \eqref{eq:  pdADMM-G q optimality} to obtain Inequality \eqref{eq:pdADMM-G objective decrease}, and we sum up Inequalities \eqref{eq: pdADMM-G-q p optimality}, \eqref{eq: pdADMM-G-q W optimality}, \eqref{eq: pdADMM-G-q b optimality}, \eqref{eq: pdADMM-G-q z optimality}, and \eqref{eq:  pdADMM-G-q q optimality} to obtain Inequality \eqref{eq:pdADMM-G-q objective decrease}.
\end{proof}
\subsubsection{Proof of Lemma \ref{lemma:lower bounded}}
\begin{proof}
(1)
There exists $\textbf{q}^{'}$ such that $p^k_{l+1}=q^{'}_l$ and 
\begin{align*}
    F(\textbf{p}^k,\textbf{W}^k,\textbf{b}^k,\textbf{z}^k,\textbf{q}^{'})\geq \min\nolimits_{\textbf{p},\textbf{W},\textbf{b},\textbf{z},\textbf{q}}\{F(\textbf{p},\textbf{W},\textbf{b},\textbf{z},\textbf{q})| p_{l+1}=q_l\}>-\infty.
\end{align*}
Therefore, we have
\begin{align*}
    &L_\rho(\textbf{p}^k,\textbf{W}^k,\textbf{b}^k,\textbf{z}^k,\textbf{q}^{k},\textbf{u}^k)=F(\textbf{p}^k,\textbf{W}^k,\textbf{b}^k,\textbf{z}^k,\textbf{q}^{k})+\sum\nolimits_{l=1}^L (u^k_l)^T(p^k_{l+1}-q^k_l)+(\rho/2)\Vert p^k_{l+1}-q^k_l \Vert^2_2\\&=R(z^k_L;y)+(\nu/2)(\sum\nolimits_{l=1}^{L}\Vert z^k_l-W^k_lp^k_l-b^k_l\Vert^2_2+\sum\nolimits_{l=1}^{L-1}\Vert q^k_l-f_l(z^k_l)\Vert^2_2)\\&+\sum\nolimits_{l=1}^{L-1}( (u^k_l)^T(p^k_{l+1}-q^k_l)+(\rho/2)\Vert p^k_{l+1}-q^k_l \Vert^2_2)\\&=R(z^k_L;y)+(\nu/2)(\sum\nolimits_{l=1}^{L}\Vert z^k_l-W^k_lp^k_l-b^k_l\Vert^2_2+\sum\nolimits_{l=1}^{L-1}\Vert q^k_l-f_l(z^k_l)\Vert^2_2)\\&+\sum\nolimits_{l=1}^{L-1}( \nu(q^k_l-f_l(z^k_l))^T(q^{'}_l-q^k_l)+(\rho/2)\Vert p^k_{l+1}-q^k_l \Vert^2_2)\\&\text{($p^k_{l+1}=q^{'}_l$ and Lemma \ref{lemma: q_opt})}\\&\geq R(z^k_L;y)+(\nu/2)(\sum\nolimits_{l=1}^{L}\Vert z^k_l-W^k_lp^k_l-b^k_l\Vert^2_2+\sum\nolimits_{l=1}^{L-1}\Vert q^{'}_l-f_l(z^k_l)\Vert^2_2)\\&-\sum\nolimits_{l=1}^{L-1}(\nu/2)\Vert q^{'}_l-q^k_l\Vert^2_2 +\sum\nolimits_{l=1}^{L-1}(\rho/2)\Vert p^k_{l+1}-q^k_l \Vert^2_2)\\&\text{($ \nu(q_l-f_l(z^{k+1}_l))=(\nu/2)\nabla_{q_l}\Vert q_l-f_l(z^{k+1}_l)\Vert^2_2$ is lipschitz continuous concerning $q_l$ and Lemma 2.1 in \cite{beck2009fast})}\\&= F(\textbf{p}^k,\textbf{W}^k,\textbf{b}^k,\textbf{z}^k,\textbf{q}^{'})+((\rho-\nu)/2)\Vert p^k_{l+1}-q^k_l\Vert^2_2> -\infty.
\end{align*}
Therefore, $F(\textbf{p}^k,\textbf{W}^k,\textbf{b}^k,\textbf{z}^k,\textbf{q}^{'})$ and $((\rho-\nu)/2)\Vert p^k_{l+1}-q^k_l\Vert^2_2$ are upper bounded by $L_\rho(\textbf{p}^k,\textbf{W}^k,\textbf{b}^k,\textbf{z}^k,\textbf{q}^{k},\textbf{u}^k)$ and hence $L_\rho(\textbf{p}^0,\textbf{W}^0,\textbf{b}^0,\textbf{z}^0,\textbf{q}^{0},\textbf{u}^0)$ (Lemma \ref{lemma:objective decrease}).
From Assumption \ref{ass:lipschitz continuous}, $(\textbf{p}^k,\textbf{W}^k,\textbf{b}^k,\textbf{z}^k)$ is bounded. $\textbf{q}^k$ is also bounded because $(\rho-\nu)/2\Vert p^k_{l+1}-q^k_l\Vert^2_2$ is upper bounded. $\textbf{u}^k$ is bounded because of Lemma \ref{lemma: q_opt}.\\
(2). It follows the same routine as (1).
\end{proof}
\subsubsection{Proof of Theorem \ref{theo:convergent variable}}
\begin{proof}
(1). From Lemmas \ref{lemma:objective decrease} and \ref{lemma:lower bounded}, we know that $L_\rho(\textbf{p}^{k},\textbf{W}^{k},\textbf{b}^k,\textbf{z}^{k},\textbf{q}^k,\textbf{u}^{k})$ is convergent because a monotone bounded sequence converges. Moreover, we take the limit on both sides of Inequality \eqref{eq:pdADMM-G objective decrease} to obtain
\begin{align*}
     &0=\lim\nolimits_{k\rightarrow\infty}L_\rho(\textbf{p}^{k},\textbf{W}^{k},\textbf{b}^k,\textbf{z}^{k},\textbf{q}^k,\textbf{u}^{k})\\&-\lim\nolimits_{k\rightarrow\infty}L_\rho(\textbf{p}^{k+1},\textbf{W}^{k+1},\textbf{b}^{k+1},\textbf{z}^{k+1},\textbf{q}^{k+1},\textbf{u}^{k+1})\\ &\geq \lim\nolimits_{k\rightarrow\infty}(\sum\nolimits_{l=2}^L (\tau^{k+1}_l/2)\Vert p^{k+1}_l\!-\!p^k_l\Vert^2_2\\&+\sum\nolimits_{l\!=\!1}^{L}(\theta^{k\!+\!1}_l/2)\Vert  W^{k\!+\!1}_l\!-\!W^k_l\Vert^2_2+\sum\nolimits_{l\!=\!1}^{L}(\nu/2)\Vert  b^{k\!+\!1}_l\!-\!b^k_l\Vert^2_2\\&+\sum\nolimits_{l=1}^{L-1} C_1\Vert z^{k+1}_l-z^k_l\Vert^2_2+(\nu/2)\Vert z^{k+1}_L-z^k_L\Vert^2_2\\&+\sum\nolimits_{l=1}^{L-1}C_2\Vert q^{k+1}_l-q^k_l\Vert^2_2)\geq 0.
\end{align*}
Because $L_\rho(\textbf{p}^{k},\textbf{W}^{k},\textbf{b}^k,\textbf{z}^{k},\textbf{q}^k,\textbf{u}^{k})$ is convergent, then $\lim_{k\rightarrow\infty}\Vert\textbf{p}^{k+1}-\textbf{p}^{k}\Vert^2_2=0$, $\lim_{k\rightarrow\infty}\Vert\textbf{W}^{k+1}-\textbf{W}^{k}\Vert^2_2=0$, $\lim_{k\rightarrow\infty}\Vert\textbf{b}^{k+1}-\textbf{b}^{k}\Vert^2_2=0$, $\lim_{k\rightarrow\infty}\Vert\textbf{z}^{k+1}-\textbf{z}^{k}\Vert^2_2=0$, and $\lim_{k\rightarrow\infty}\Vert\textbf{q}^{k+1}-\textbf{q}^{k}\Vert^2_2=0$. $\lim_{k\rightarrow\infty}\Vert\textbf{u}^{k+1}-\textbf{u}^{k}\Vert^2_2=0$ is derived from Lemma \ref{lemma:u square bound} in Section \ref{sec:convergence proof} in the Appendix.\\
(2). The proof follows the same procedure as (1).
\end{proof}

\subsubsection{Proof of Lemma \ref{lemma:subgradient bound}}
\begin{proof} (1). We know that
$\partial L_\rho(\textbf{p}^{k+1},\textbf{W}^{k+1},\textbf{b}^{k+1},\textbf{z}^{k+1},\textbf{q}^{k+1},\textbf{u}^{k+1})=\{\nabla_{\text{p}^{k+1}}L_\rho,\nabla_{\text{W}^{k+1}}L_\rho,\nabla_{\text{b}^{k+1}}L_\rho,\partial_{\text{z}^{k+1}}L_\rho,\nabla_{\text{q}^{k+1}}L_\rho,\nabla_{\text{u}^{k+1}}L_\rho\}$ \cite{wang2019admm}. Specifically, we prove that $\Vert g\Vert$ is upper bounded by the linear combination of $\Vert\textbf{p}^{k+1}-\textbf{p}^{k}\Vert$,$ \Vert\textbf{W}^{k+1}-\textbf{W}^{k}\Vert$, $\Vert\textbf{b}^{k+1}-\textbf{b}^{k}\Vert$, $\Vert\textbf{z}^{k+1}-\textbf{z}^{k}\Vert$, $\Vert\textbf{q}^{k+1}-\textbf{q}^{k}\Vert$, and $\Vert\textbf{u}^{k+1}-\textbf{u}^{k}\Vert$.\\
For $p^{k+1}_l$,
\begin{align*}
    &\nabla_{p^{k+1}_l}
L_\rho(\textbf{p}^{k+1},\textbf{W}^{k+1},\textbf{b}^{k+1},\textbf{z}^{k+1},\textbf{q}^{k+1},\textbf{u}^{k+1})\\&=\nabla_{p^{k+1}_l}\phi({p}^{k+1}_l,{W}^{k+1}_l,{b}^{k+1}_l,{z}^{k+1}_l,q^{k+1}_{l-1},u^{k+1}_{l-1})\\&=\nabla_{p^{k}_l}\phi({p}^{k}_l,{W}^{k}_l,{b}^{k}_l,{z}^{k}_l,q^{k}_{l-1},u^{k}_{l-1})+\tau^{k+1}_l(p^{k+1}_l-p^k_l)-\tau^{k+1}_l(p^{k+1}_l-p^k_l)\\&+\nu(W^{k+1}_l)^T W^{k+1}_l p^{k+1}_l-\nu(W^{k}_l)^T W^{k}_l p^{k}_l+\nu(W^{k+1}_l)^T b^{k+1}_l-\nu(W^{k}_l)^T b^{k}_l-\nu(W^{k+1}_l)^T z^{k+1}_l+\nu(W^{k}_l)^T z^{k}_l\\&+(u^{k+1}_{l-1}-u^{k}_{l-1})+\rho(p^{k+1}_l-p^{k}_l)-\rho(q^{k+1}_{l-1}-q^{k}_{l-1})\\&=-\tau^{k+1}_l(p^{k+1}_l-p^k_l)+\nu(W^{k+1}_l)^T W^{k+1}_l p^{k+1}_l-\nu(W^{k}_l)^T W^{k}_l p^{k}_l+\nu(W^{k+1}_l)^T b^{k+1}_l-\nu(W^{k}_l)^T b^{k}_l\\&-\nu(W^{k+1}_l)^T z^{k+1}_l+\nu(W^{k}_l)^T z^{k}_l+(u^{k+1}_{l-1}-u^{k}_{l-1})+\rho(p^{k+1}_l-p^{k}_l)-\rho(q^{k+1}_{l-1}-q^{k}_{l-1}).
\end{align*}
So
\begin{align*}
    &\Vert \nabla_{p^{k+1}_l}
L_\rho(\textbf{p}^{k+1},\textbf{W}^{k+1},\textbf{b}^{k+1},\textbf{z}^{k+1},\textbf{q}^{k+1},\textbf{u}^{k+1})\Vert\\&=\Vert \tau^{k+1}_l(p^{k+1}_l-p^k_l)+\nu(W^{k+1}_l)^T W^{k+1}_l p^{k+1}_l-\nu(W^{k}_l)^T W^{k}_l p^{k}_l+\nu(W^{k+1}_l)^T b^{k+1}_l-\nu(W^{k}_l)^T b^{k}_l\\&-\nu(W^{k+1}_l)^T z^{k+1}_l+\nu(W^{k}_l)^T z^{k}_l+(u^{k+1}_{l-1}-u^{k}_{l-1})+\rho(p^{k+1}_l-p^{k}_l)-\rho(q^{k+1}_{l-1}-q^{k}_{l-1})\Vert\\&\leq \tau^{k+1}_l\Vert p^{k+1}_l-p^k_l\Vert+\nu\Vert(W^{k+1}_l)^T W^{k+1}_l p^{k+1}_l-(W^{k}_l)^T W^{k}_l p^{k}_l\Vert+\nu\Vert(W^{k+1}_l)^T b^{k+1}_l-(W^{k}_l)^T b^{k}_l\Vert\\&+\nu\Vert (W^{k+1}_l)^T z^{k+1}_l-(W^{k}_l)^T z^{k}_l\Vert+\Vert u^{k+1}_{l-1}-u^{k}_{l-1}\Vert+\rho\Vert p^{k+1}_l-p^{k}_l\Vert+\rho\Vert q^{k+1}_{l-1}-q^{k}_{l-1}\Vert \ \text{(triangle inequality)}\\&=\tau^{k+1}_l\Vert p^{k+1}_l-p^k_l\Vert+\nu\Vert(W^{k+1}_l)^T W^{k+1}_l (p^{k+1}_l-p^k_l)+(W^{k+1}_l)^T(W^{k+1}_l-W^{k}_l)p^k_l+(W^{k+1}_l-W^{k}_l)^T W^{k}_l p^{k}_l\Vert\\&+\nu\Vert(W^{k+1}_l)^T (b^{k+1}_l-b^k_l)+(W^{k+1}_l-W^{k}_l)^T b^{k}_l\Vert+\nu\Vert(W^{k+1}_l)^T (z^{k+1}_l-z^k_l)+(W^{k+1}_l-W^{k}_l)^T z^{k}_l\Vert\\&+\Vert u^{k+1}_{l-1}-u^k_{l-1}\Vert+\rho\Vert p^{k+1}_l-p^{k}_l\Vert+\rho\Vert q^{k+1}_{l-1}-q^{k}_{l-1}\Vert\\&\leq \tau^{k+1}_l\Vert p^{k+1}_l-p^k_l\Vert+\nu\Vert W^{k+1}_l\Vert^2\Vert p^{k+1}_l-p^k_l\Vert+\nu\Vert W^{k+1}_l\Vert \Vert W^{k+1}_l-W^{k}_l\Vert \Vert p^k_l\Vert+\nu\Vert W^{k+1}_l-W^{k}_l\Vert\Vert W^{k}_l\Vert\Vert p^{k}_l\Vert\\&+\nu\Vert W^{k+1}_l\Vert \Vert b^{k+1}_l-b^k_l\Vert +\nu\Vert W^{k+1}_l-W^{k}_l\Vert\Vert b^{k}_l\Vert+\nu\Vert W^{k+1}_l \Vert \Vert z^{k+1}_l-z^k_l\Vert+\nu\Vert W^{k+1}_l-W^{k}_l\Vert \Vert z^{k}_l\Vert\\&+\nu(\Vert q^{k+1}_{l-1}-q^{k}_{l-1}\Vert+S\Vert z^{k+1}_{l-1}-z^{k}_{l-1}\Vert)+\rho\Vert p^{k+1}_l-p^{k}_l\Vert+\rho\Vert q^{k+1}_{l-1}-q^{k}_{l-1}\Vert\\&\text{(triangle inequality, Cauthy-Schwartz inequality and Lemma \ref{lemma:u bound})}\\&\leq \tau^{k+1}_l\Vert p^{k+1}_l-p^k_l\Vert+\nu\mathbb{N}^2_\textbf{W}\Vert p^{k+1}_l-p^k_l\Vert+2\nu\mathbb{N}_\textbf{W}\mathbb{N}_\textbf{p} \Vert W^{k+1}_l-W^{k}_l\Vert+\nu\mathbb{N}_\textbf{W} \Vert b^{k+1}_l-b^k_l\Vert +\nu\mathbb{N}_\textbf{b}\Vert W^{k+1}_l-W^{k}_l\Vert\\&+\nu\mathbb{N}_\textbf{W} \Vert z^{k+1}_l-z^k_l\Vert+\nu\mathbb{N}_\textbf{z}\Vert W^{k+1}_l-W^{k}_l\Vert+2\nu^2(\Vert q^{k+1}_{l-1}-q^{k}_{l-1}\Vert^2_2+S^2\Vert z^{k+1}_{l-1}-z^{k}_{l-1}\Vert^2_2)+\rho\Vert p^{k+1}_l-p^{k}_l\Vert+\rho\Vert q^{k+1}_{l-1}-q^{k}_{l-1}\Vert\\&\text{(Lemma  \ref{lemma:lower bounded})}.
\end{align*}
For $W^{k+1}_1$,
\begin{align*}
&\nabla_{W^{k+1}_1}
L_\rho(\textbf{p}^{k+1},\textbf{W}^{k+1},\textbf{b}^{k+1},\textbf{z}^{k+1},\textbf{q}^{k+1},\textbf{u}^{k+1})\\&=\nabla_{W^{k+1}_1} \phi({p}^{k+1}_1,{W}^{k+1}_1,{b}^{k+1}_1,{z}^{k+1}_1)\\&=\nabla_{W^{k}_1} \phi({p}^{k+1}_1,{W}^{k}_1,{b}^{k}_1,{z}^{k}_1)+\theta^{k+1}_1(W^{k+1}_1-W^k_1)+\nu(W^{k+1}_1-W^k_1)p^{k+1}_1(p^{k+1}_1)^T+\nu(b^{k+1}_1-b^k_1)(p^{k+1}_1)^T\\&-\nu(z^{k+1}_1-z^k_1)(p^{k+1}_1)^T-\theta^{k+1}_1(W^{k+1}_1-W^k_1)\\&=\nu(W^{k+1}_1-W^k_1)p^{k+1}_1(p^{k+1}_1)^T+\nu(b^{k+1}_1-b^k_1)(p^{k+1}_1)^T-\nu(z^{k+1}_1-z^k_1)(p^{k+1}_1)^T-\theta^{k+1}_1(W^{k+1}_1-W^k_1)\\ & \text{(The optimality condition of Equation \eqref{eq:update W})}.
\end{align*}
So
\begin{align*}
    &\Vert\nabla_{W^{k+1}_1}
L_\rho(\textbf{p}^{k+1},\textbf{W}^{k+1},\textbf{b}^{k+1},\textbf{z}^{k+1},\textbf{q}^{k+1},\textbf{u}^{k+1})\Vert\\&=\Vert \nu(W^{k+1}_1-W^k_1)p^{k+1}_1(p^{k+1}_1)^T+\nu(b^{k+1}_1-b^k_1)(p^{k+1}_1)^T-\nu(z^{k+1}_1-z^k_1)(p^{k+1}_1)^T-\theta^{k+1}_1(W^{k+1}_1-W^k_1)\Vert\\&\leq \nu\Vert W^{k+1}_1-W^k_1\Vert\Vert p^{k+1}_1\Vert^2+\nu\Vert b^{k+1}_1-b^k_1\Vert\Vert p^{k+1}_1\Vert+\nu\Vert z^{k+1}_1-z^k_1\Vert\Vert p^{k+1}_1\Vert+\theta^{k+1}_1\Vert W^{k+1}_1-W^k_1\Vert\\&\text{(triangle inequality and Cauthy-Schwartz inequality)}\\&\leq \nu\Vert W^{k+1}_1-W^k_1\Vert \mathbb{N}_\textbf{p}^2+\nu\Vert b^{k+1}_1-b^k_1\Vert\mathbb{N}_\textbf{p}+\nu\Vert z^{k+1}_1-z^k_1\Vert\mathbb{N}_\textbf{p}+\theta^{k+1}_1\Vert W^{k+1}_1-W^k_1\Vert\ \text{(Theorem \ref{theo:convergent variable})}.
\end{align*}
For $W^{k+1}_l(1<l\leq L)$,
\begin{align*}
&\nabla_{W^{k+1}_l}
L_\rho(\textbf{p}^{k+1},\textbf{W}^{k+1},\textbf{b}^{k+1},\textbf{z}^{k+1},\textbf{q}^{k+1},\textbf{u}^{k+1})\\&=\nabla_{W^{k+1}_l} \phi({p}^{k+1}_l,{W}^{k+1}_l,{b}^{k+1}_l,{z}^{k+1}_l,p^{k+1}_{l-1},u^{k+1}_{l-1})\\&=\nabla_{W^{k}_l} \phi({p}^{k+1}_l,{W}^{k}_l,{b}^{k}_l,{z}^{k}_l,p^{k}_{l-1},u^{k}_{l-1})+\theta^{k+1}_l(W^{k+1}_l-W^k_1)+\nu(W^{k+1}_l-W^k_l)p^{k+1}_l(p^{k+1}_l)^T\\&+\nu(b^{k+1}_l-b^k_l)(p^{k+1}_l)^T-\nu(z^{k+1}_l-z^k_l)(p^{k+1}_l)^T-\theta^{k+1}_l(W^{k+1}_l-W^k_l)\\&=\nu(W^{k+1}_l-W^k_l)p^{k+1}_l(p^{k+1}_l)^T+\nu(b^{k+1}_l-b^k_l)(p^{k+1}_l)^T-\nu(z^{k+1}_l-z^k_l)(p^{k+1}_l)^T-\theta^{k+1}_l(W^{k+1}_l-W^k_l)\\ & \text{(The optimality condition of Equation \eqref{eq:update W})}.
\end{align*}
So
\begin{align*}
    &\Vert\nabla_{W^{k+1}_l}
L_\rho(\textbf{p}^{k+1},\textbf{W}^{k+1},\textbf{b}^{k+1},\textbf{z}^{k+1},\textbf{q}^{k+1},\textbf{u}^{k+1})\Vert\\&=\Vert \nu(W^{k+1}_l-W^k_l)p^{k+1}_l(p^{k+1}_l)^T+\nu(b^{k+1}_l-b^k_l)(p^{k+1}_l)^T-\nu(z^{k+1}_l-z^k_l)(p^{k+1}_l)^T-\theta^{k+1}_l(W^{k+1}_l-W^k_l)\Vert\\&\leq \nu\Vert W^{k+1}_l-W^k_l\Vert\Vert p^{k+1}_l\Vert^2+\nu\Vert b^{k+1}_l-b^k_l\Vert\Vert p^{k+1}_l\Vert+\nu\Vert z^{k+1}_l-z^k_l\Vert\Vert p^{k+1}_l\Vert+\theta^{k+1}_l\Vert W^{k+1}_l-W^k_l\Vert\\&\text{(triangle inequality and Cauthy-Schwartz inequality)}\\&\leq \nu\Vert W^{k+1}_l-W^k_l\Vert \mathbb{N}_\textbf{p}^2+\nu\Vert b^{k+1}_l-b^k_l\Vert\mathbb{N}_\textbf{p}+\nu\Vert z^{k+1}_l-z^k_l\Vert\mathbb{N}_\textbf{p}+\theta^{k+1}_l\Vert W^{k+1}_l-W^k_l\Vert\ \text{(Theorem \ref{theo:convergent variable})}.
\end{align*}
For $b^{k+1}_1$,
\begin{align*}
 &\nabla_{b^{k+1}_1}
L_\rho(\textbf{p}^{k+1},\textbf{W}^{k+1},\textbf{b}^{k+1},\textbf{z}^{k+1},\textbf{q}^{k+1},\textbf{u}^{k+1})\\&= \nabla_{b^{k+1}_1}
\phi(p^{k+1}_1,W^{k+1}_1,b^{k+1}_1,z^{k+1}_1)\\&= \nabla_{b^{k}_1}
\phi(p^{k+1}_1,W^{k+1}_1,b^{k}_1,z^{k}_1)+\nu(b^{k+1}_1-b^k_1)+\nu(z^k_1-z^{k+1}_1)\\&=\nu(z^k_1-z^{k+1}_1) \ \text{(The optimality condition of Equation \eqref{eq:update b})}.
\end{align*}
So $\Vert \nabla_{b^{k+1}_1}
L_\rho(\textbf{p}^{k+1},\textbf{W}^{k+1},\textbf{b}^{k+1},\textbf{z}^{k+1},\textbf{q}^{k+1},\textbf{u}^{k+1})\Vert=\nu\Vert z^{k+1}_1-z^k_1\Vert$.\\
For $b^{k+1}_l(1<l\leq L)$,
\begin{align*}
 &\nabla_{b^{k+1}_l}
L_\rho(\textbf{p}^{k+1},\textbf{W}^{k+1},\textbf{b}^{k+1},\textbf{z}^{k+1},\textbf{q}^{k+1},\textbf{u}^{k+1})\\&= \nabla_{b^{k+1}_l}
\phi(p^{k+1}_l,W^{k+1}_l,b^{k+1}_l,z^{k+1}_l,q^{k}_{l-1},u^{k}_{l-1})\\&= \nabla_{b^{k}_l}
\phi(p^{k+1}_l,W^{k+1}_l,b^{k}_l,z^{k}_l,q^{k}_{l-1},u^{k}_{l-1})+\nu(b^{k+1}_l-b^k_l)+\nu(z^k_l-z^{k+1}_l)\\&=\nu(z^k_l-z^{k+1}_l) \ \text{(The optimality condition of Equation \eqref{eq:update b})}.
\end{align*}
So $\Vert \nabla_{b^{k+1}_l}
L_\rho(\textbf{p}^{k+1},\textbf{W}^{k+1},\textbf{b}^{k+1},\textbf{z}^{k+1},\textbf{q}^{k+1},\textbf{u}^{k+1})\Vert=\nu\Vert z^{k+1}_l-z^k_l\Vert$.\\
For $z^{k+1}_l(l<L)$, 
\begin{align*}
&\partial_{z^{k+1}_l}
L_\rho(\textbf{p}^{k+1},\textbf{W}^{k+1},\textbf{b}^{k+1},\textbf{z}^{k+1},\textbf{q}^{k+1},\textbf{u}^{k+1})\\&=\partial_{z^{k+1}_l}
L_\rho(\textbf{p}^{k+1},\textbf{W}^{k+1},\textbf{b}^{k+1},\textbf{z}^{k+1},\textbf{q}^{k},\textbf{u}^{k})+\nu(z_l^{k+1}-z_l^{k})-\nu(z_l^{k+1}-z_l^{k})-\nu\partial f_l(z^{k+1}_l)\circ(q^{k+1}_l-q^k_l)\text{($\circ$ is Hadamard product)}\\&=-\nu(z_l^{k+1}-z_l^{k})-\nu\partial f_l(z^{k+1}_l)\circ(q^{k+1}_l-q^k_l) \ \text{($ 0\in \partial_{z^{k+1}_l}
L_\rho(\textbf{p}^{k+1},\textbf{W}^{k+1},\textbf{b}^{k+1},\textbf{z}^{k+1},\textbf{q}^{k},\textbf{u}^{k})+\nu(z_l^{k+1}-z_l^{k})$)}.
\end{align*}
So
\begin{align*}
   &\Vert\partial_{z^{k+1}_l}
L_\rho(\textbf{p}^{k+1},\textbf{W}^{k+1},\textbf{b}^{k+1},\textbf{z}^{k+1},\textbf{q}^{k+1},\textbf{u}^{k+1})\Vert\\&= \Vert -\nu(z_l^{k+1}-z_l^{k})-\nu\partial f_l(z^{k+1}_l)\circ(q^{k+1}_l-q^k_l)\Vert\\&\leq \nu\Vert z^{k+1}_l-z^k_l\Vert+\nu\Vert \partial f_l(z^{k+1}_l)\Vert\Vert q_l^{k+1}-q_l^k\Vert \text{(Cauchy-Schwartz inequality and triangle inequality)}\\&\leq \nu\Vert z^{k+1}_l-z^k_l\Vert+\nu M\Vert q_l^{k+1}-q_l^k\Vert (\text{$\Vert \partial f_l(z^{k+1}_l)\Vert\leq M$}).
\end{align*}
For $z^{k+1}_L$,
$\partial_{z^{k+1}_L}
L_\rho(\textbf{p}^{k+1},\textbf{W}^{k+1},\textbf{b}^{k+1},\textbf{z}^{k+1},\textbf{q}^{k+1},\textbf{u}^{k+1})=0$ by the optimality condition of Equation \eqref{eq:update zl}.\\
For $q^{k+1}_l$,
\begin{align*}
&\nabla_{q^{k+1}_l}
L_\rho(\textbf{p}^{k+1},\textbf{W}^{k+1},\textbf{b}^{k+1},\textbf{z}^{k+1},\textbf{q}^{k+1},\textbf{u}^{k+1})\\&=\nabla_{q^{k+1}_l}
L_\rho(\textbf{p}^{k+1},\textbf{W}^{k+1},\textbf{b}^{k+1},\textbf{z}^{k+1},\textbf{q}^{k+1},\textbf{u}^{k})+u^{k+1}_l-u^{k}_l\\&=u^{k+1}_l-u^{k}_l \ \text{($\nabla_{q^{k+1}_l}
L_\rho(\textbf{p}^{k+1},\textbf{W}^{k+1},\textbf{b}^{k+1},\textbf{z}^{k+1},\textbf{q}^{k+1},\textbf{u}^{k})=0$ by the optimality condition of Equation \eqref{eq:update q})}.
\end{align*}
So $\Vert\nabla_{q^{k+1}_l}
L_\rho(\textbf{p}^{k+1},\textbf{W}^{k+1},\textbf{b}^{k+1},\textbf{z}^{k+1},\textbf{q}^{k+1},\textbf{u}^{k+1})\Vert=\Vert u^{k+1}_l-u^{k}_l\Vert$.\\
For $u^{k+1}_l$,
\begin{align*}
&\nabla_{u^{k+1}_l} L_\rho(\textbf{p}^{k+1},\textbf{W}^{k+1},\textbf{b}^{k+1},\textbf{z}^{k+1},\textbf{q}^{k+1},\textbf{u}^{k+1})=(p^{k+1}_{l+1}-q^{k+1}_{l})=(u^{k+1}_{l}-u^{k}_{l})/\rho.
\end{align*}
So $\Vert\nabla_{u^{k+1}_l} L_\rho(\textbf{p}^{k+1},\textbf{W}^{k+1},\textbf{b}^{k+1},\textbf{z}^{k+1},\textbf{q}^{k+1},\textbf{u}^{k+1})\Vert=\Vert u^{k+1}_{l}-u^{k}_{l}\Vert/\rho$.\\
In summary, we prove that $\nabla_{\text{p}^{k+1}}L_\rho,\nabla_{\text{W}^{k+1}}L_\rho,\nabla_{\text{b}^{k+1}}L_\rho,\partial_{\text{z}^{k+1}}L_\rho,\nabla_{\text{q}^{k+1}}L_\rho,\nabla_{\text{u}^{k+1}}L_\rho$ are upper bounded by the linear combination of $\Vert\textbf{p}^{k+1}-\textbf{p}^{k}\Vert$,$ \Vert\textbf{W}^{k+1}-\textbf{W}^{k}\Vert$, $\Vert\textbf{b}^{k+1}-\textbf{b}^{k}\Vert$, $\Vert\textbf{z}^{k+1}-\textbf{z}^{k}\Vert$, $\Vert\textbf{q}^{k+1}-\textbf{q}^{k}\Vert$, and $\Vert\textbf{u}^{k+1}-\textbf{u}^{k}\Vert$.\\
(2). It follows exactly the proof of (1) except for $p^{k+1}_l$ .
\end{proof}
\subsubsection{The proof of Theorem \ref{theo: pdADMM-G global convergence}}
\begin{proof}
 From Lemma \ref{lemma:lower bounded}(1), $(\textbf{p}^k,\textbf{W}^k, \textbf{b}^k, \textbf{z}^k, \textbf{q}^k,\textbf{u}^k)$ has at least a limit point $(\textbf{p}^*,\textbf{W}^*, \textbf{b}^*, \textbf{z}^*, \textbf{q}^*,\textbf{u}^*)$ because a bounded sequence has at least a limit point. From Lemma \ref{lemma:subgradient bound} and Theorem \ref{theo:convergent variable},
$\Vert g^{k+1}\Vert \rightarrow 0$ as $k\rightarrow \infty$. According to the definition of general subgradient (Definition 8.3 in \cite{rockafellar2009variational}), we have $0\in \partial L_\rho(\textbf{p}^*,\textbf{W}^*,\textbf{b}^*,\textbf{z}^*,\textbf{q}^*,\textbf{u}^*)$. In other words, every limit point $(\textbf{p}^*,\textbf{W}^*, \textbf{b}^*, \textbf{z}^*, \textbf{q}^*,\textbf{u}^*)$ is a stationary point.
\end{proof}

\subsubsection{The proof of Theorem \ref{theo: pdADMM-G-q global convergence}}
\begin{proof}
 From Lemma \ref{lemma:lower bounded}(2), $(\textbf{W}^k, \textbf{b}^k, \textbf{z}^k, \textbf{q}^k,\textbf{u}^k)$ has at least a limit point $(\textbf{W}^*, \textbf{b}^*, \textbf{z}^*, \textbf{q}^*,\textbf{u}^*)$ because a bounded sequence has at least a limit point. $\textbf{p}^k$ has at least a limit point $\textbf{p}^*$ because $\textbf{p}^k\in\Delta$ and $\Delta$ is finite. From Lemma \ref{lemma:subgradient bound}(2) and Theorem \ref{theo:convergent variable}, $\Vert\overline{g}_\textbf{W}^{k+1}\Vert\rightarrow 0$, $\Vert\overline{g}_\textbf{b}^{k+1}\Vert\rightarrow 0$, $\Vert\overline{g}_\textbf{z}^{k+1}\Vert\rightarrow 0$,
 $\Vert\overline{g}_\textbf{q}^{k+1}\Vert \rightarrow 0$,
 $\Vert\overline{g}_\textbf{u}^{k+1}\Vert \rightarrow 0$ as $k\rightarrow \infty$. According to the definition of general subgradient (Defintion 8.3 in \cite{rockafellar2009variational}), we have $\nabla_{\textbf{W}^*} \beta_\rho(\textbf{p}^*,\textbf{W}^*,\textbf{b}^*,\textbf{z}^*,\textbf{q}^*,\textbf{u}^*)=0$, $\nabla_{\textbf{b}^*} \beta_\rho(\textbf{p}^*,\textbf{W}^*,\textbf{b}^*,\textbf{z}^*,\textbf{q}^*,\textbf{u}^*)=0$, $ 0\in\partial_{\textbf{z}^*} \beta_\rho(\textbf{p}^*,\textbf{W}^*,\textbf{b}^*,\textbf{z}^*,\textbf{q}^*,\textbf{u}^*)$, $\nabla_{\textbf{q}^*} \beta_\rho(\textbf{p}^*,\textbf{W}^*,\textbf{b}^*,\textbf{z}^*,\textbf{q}^*,\textbf{u}^*)=0$ and $\nabla_{\textbf{u}^*} \beta_\rho(\textbf{p}^*,\textbf{W}^*,\textbf{b}^*,\textbf{z}^*,\textbf{q}^*,\textbf{u}^*)=0$ (i.e. $p^*_{l+1}=q^*_l$). In other words, every limit point $(\textbf{W}^*, \textbf{b}^*, \textbf{z}^*,\textbf{u}^*)$ is a stationary point of Problem \ref{prob:problem 3}. Moreover, $\tau^k_l$ has a limit point $\tau^*_l$ because it is bounded. Let $\boldsymbol{\tau}^k=\{\tau^k_l\}_{l=2}^L$. Consider a subsequence $(\textbf{p}^s,\textbf{W}^s,\textbf{b}^s,\textbf{z}^s,\textbf{q}^s,\textbf{u}^s,\boldsymbol{\tau}^{s+1})\rightarrow (\textbf{p}^*,\textbf{W}^*,\textbf{b}^*,\textbf{z}^*,\textbf{q}^*,\textbf{u}^*,\boldsymbol{\tau}^*)$. Because $u^{s+1}_l=u_l^s+\rho(p^s_{l+1}-q^s_{l})$ and $u^{s+1}_l\rightarrow u^{s}_l$, thus $p^s_{l+1}\rightarrow q^s_{l}$, and $p^{s+1}_{l+1}\rightarrow q^{s+1}_{l}$. Because $q^{s+1}_{l}\rightarrow q^{s}_{l}$, then $p^{s+1}_{l+1}\rightarrow p^{s}_{l+1}$ for any $l$. In other words, $\textbf{p}^{s+1}\rightarrow \textbf{p}^{s}$. Because $\textbf{p}^{s}\rightarrow \textbf{p}^*$, then $\textbf{p}^{s+1}\rightarrow \textbf{p}^*$.
 The optimality condition of $\textbf{p}^{s+1}$ (i.e. Equation \eqref{eq:update p quantized}) leads to 
 \begin{align*}
     p^{s+1}_l\leftarrow \arg\min\nolimits_{\delta\in\Delta} \Vert \delta- p^s_{l}-\nabla_{p^s_{l}}\phi(p^s_l,W^s_l,b^s_l,z^s_l,q^s_{l-1},u^s_{l-1})/{\tau}^{s+1}_{l})\Vert.
 \end{align*}
 Taking $s\rightarrow\infty$ on both sides, we have
 \begin{align*}
     p^*_l\leftarrow \arg\min\nolimits_{\delta\in\Delta} \Vert \delta- (p^*_{l}-\nabla_{p^*_{l}}\phi(p^*_l,W^*_l,b^*_l,z^*_l,q^*_{l-1},u^*_{l-1})/{\tau}^*_{l})\Vert.
 \end{align*}
 Because $\nabla_{p_l^*}F(\textbf{p}^*,\textbf{W}^*,\textbf{b}^*,\textbf{z}^*,\textbf{q}^*)=\nu W^T_l(z^*_l-W^*_lp^*_l-b^*_l)=\nabla_{p^*_{l}}\phi(p^*_l,W^*_l,b^*_l,z^*_l,q^*_{l-1},u^*_{l-1})$. Then
 \begin{align*}
     p^*_l\leftarrow \arg\min\nolimits_{\delta\in\Delta} \Vert \delta- (p^*_{l}-\nabla_{p_l^*}F(\textbf{p}^*,\textbf{W}^*,\textbf{b}^*,\textbf{z}^*,\textbf{q}^*)/{\tau}^*_{l})\Vert.
 \end{align*}
 Namely, $\textbf{p}^*$ is a quantized stationary point of Problem \ref{prob:problem 3}.
\end{proof}

\subsubsection{The proof of Theorem \ref{theo: convergence rate}}
\begin{proof}
(1). To prove this, we will first show that $c_k$ satisfies two conditions: (1). $c_k\geq c_{k+1}$. (2). $\sum\nolimits_{k=0}^\infty c_k$ is bounded.  We then conclude the convergence rate of $o(1/k)$ based on these two conditions. Specifically, first, we have
\begin{align*}
    c_k&\!=\!\min\nolimits_{0\leq i\leq k}\!(\sum\nolimits_{l=2}^L (\tau^{i+1}_l/2)\Vert p^{i+1}_l\!-\!p^i_l\Vert^2_2+\sum\nolimits_{l\!=\!1}^{L}(\theta^{i\!+\!1}_l/2)\Vert  W^{i\!+\!1}_l\!-\!W^i_l\Vert^2_2+\sum\nolimits_{l\!=\!1}^{L}(\nu/2)\Vert  b^{i\!+\!1}_l\!-\!b^i_l\Vert^2_2\\&+\sum\nolimits_{l=1}^{L-1} C_1\Vert z^{i+1}_l-z^i_l\Vert^2_2+(\nu/2)\Vert z^{i+1}_L-z^i_L\Vert^2_2+\sum\nolimits_{l=1}^{L-1}C_2\Vert q^{i+1}_l-q^i_l\Vert^2_2) \\&\geq\!\min\nolimits_{0\leq i\leq k+1}\!(\sum\nolimits_{l=2}^L (\tau^{i+1}_l/2)\Vert p^{i+1}_l\!-\!p^i_l\Vert^2_2+\sum\nolimits_{l\!=\!1}^{L}(\theta^{i\!+\!1}_l/2)\Vert  W^{i\!+\!1}_l\!-\!W^i_l\Vert^2_2+\sum\nolimits_{l\!=\!1}^{L}(\nu/2)\Vert  b^{i\!+\!1}_l\!-\!b^i_l\Vert^2_2\\&+\sum\nolimits_{l=1}^{L-1} C_1\Vert z^{i+1}_l-z^i_l\Vert^2_2+(\nu/2)\Vert z^{i+1}_L-z^i_L\Vert^2_2+\sum\nolimits_{l=1}^{L-1}C_2\Vert q^{i+1}_l-q^i_l\Vert^2_2)\\&= c_{k+1}.
\end{align*}
Therefore $c_k$ satisfies the first condition. Second,
\begin{align*}
    &\sum\nolimits_{k\!=\!0}^\infty c_k\\\!&=\sum\nolimits_{k=0}^\infty\min\nolimits_{0\leq i\leq k}\!(\sum\nolimits_{l=2}^L (\tau^{i+1}_l/2)\Vert p^{i+1}_l\!-\!p^i_l\Vert^2_2+\sum\nolimits_{l\!=\!1}^{L}(\theta^{i\!+\!1}_l/2)\Vert  W^{i\!+\!1}_l\!-\!W^i_l\Vert^2_2+\sum\nolimits_{l\!=\!1}^{L}(\nu/2)\Vert  b^{i\!+\!1}_l\!-\!b^i_l\Vert^2_2\\&+\sum\nolimits_{l=1}^{L-1} C_1\Vert z^{i+1}_l-z^i_l\Vert^2_2+(\nu/2)\Vert z^{i+1}_L-z^i_L\Vert^2_2+\sum\nolimits_{l=1}^{L-1}C_2\Vert q^{i+1}_l-q^i_l\Vert^2_2)\\&\leq \sum\nolimits_{k=0}^\infty(\sum\nolimits_{l=2}^L (\tau^{k+1}_l/2)\Vert p^{k+1}_l\!-\!p^k_l\Vert^2_2+\sum\nolimits_{l\!=\!1}^{L}(\theta^{k\!+\!1}_l/2)\Vert  W^{k\!+\!1}_l\!-\!W^k_l\Vert^2_2+\sum\nolimits_{l\!=\!1}^{L}(\nu/2)\Vert  b^{k\!+\!1}_l\!-\!b^k_l\Vert^2_2\\&+\sum\nolimits_{l=1}^{L-1} C_1\Vert z^{k+1}_l-z^k_l\Vert^2_2+(\nu/2)\Vert z^{k+1}_L-z^k_L\Vert^2_2+\sum\nolimits_{l=1}^{L-1}C_2\Vert q^{k+1}_l-q^k_l\Vert^2_2) \\&\leq L_\rho(\textbf{p}^0,\textbf{W}^0,\textbf{b}^0,\textbf{z}^0,\textbf{q}^0,\textbf{u}^0)-L_\rho(\textbf{p}^*,\textbf{W}^*,\textbf{b}^*,\textbf{z}^*,\textbf{q}^*,\textbf{u}^*)\text{(Lemma \ref{lemma:objective decrease})}.
\end{align*}
So $\sum\nolimits_{k=0}^\infty c_k$ is bounded and $c_{k}$ satisfies the second condition. Finally, it has been proved that the sufficient conditions of convergence rate $o(1/k)$ are: (1) $c_k\geq c_{k+1}$, and (2) $\sum\nolimits_{k=0}^\infty c_k$ is bounded, and (3) $c_k\geq0$ (Lemma 1.2 in \cite{deng2017parallel}). Since we have proved the first two conditions and the third one $c_k \geq 0$ is obvious, the convergence rate of $o(1/k)$ is proven. \\
(2). It follows the same procedure as (1).
\end{proof}
\subsection{More Experimental Results}
\subsubsection{Datasets Details}
\label{sec:dataset}
 1. Cora \cite{sen2008collective}. The Cora dataset consists of 2708 scientific publications classified into one of seven classes. The citation network consists of 5429 links. Each publication in the dataset is described by a 0/1-valued word vector indicating the absence/presence of the corresponding word from the dictionary. The dictionary consists of 1433 unique words.\\
2. PubMed \cite{sen2008collective}. PubMed comprises 30M+ citations for biomedical literature that have been collected from sources such as MEDLINE, life science journals, and published online e-books. It also includes links to text content from PubMed Central and other publishers' websites.\\
3. Citeseer \cite{sen2008collective}. The Citeseer dataset was collected from the Tagged.com social network website. It contains 5.6 million users and 858 million links between them. Each user has 4 features and is manually labeled as ``spammer" or ``not spammer". Each link represents an action between two users and includes a timestamp and a type. The network contains 7 anonymized types of links. The original task on the dataset is to identify (i.e., classify) the spammer users based on their relational and non-relational features. \\
4. Amazon Computers and Amazon Photo \cite{mcauley2015image}. Amazon Computers and Amazon Photo are segments of the Amazon co-purchase graph, where nodes represent goods, edges indicate that two goods are frequently bought
together, node features are bag-of-words encoded product reviews, and class labels are given by the product category.\\
5. Coauthor CS and Coauthor Physics \cite{shchur2018pitfalls}. Coauthor CS and Coauthor Physics are co-authorship graphs based on the Microsoft Academic Graph from the KDD Cup 2016 challenge 3. Here, nodes are authors, that are connected by an edge if they co-authored a paper; node features represent paper keywords for each author’s papers, and class
labels indicate the most active fields of study for each author.\\
6. Flickr \cite{zeng2020graphsaint}. In Flickr, one node in the graph represents one image uploaded to Flickr. If two images share some common properties (e.g., same geographic location, same gallery, comments by the same user, etc.), there is an edge between
the nodes of these two images. Node features are bag-of-word
representation of the images and labels are classes of images.\\
7. Ogbn-Arxiv \cite{hu2020open}. The Ogbn-Arxiv dataset is a directed graph, representing the citation network between all Computer Science (CS) ARXIV papers indexed by MAG. Each node is an ARXIV paper and each directed edge indicates that one paper cites another one. Each paper comes with a
128-dimensional feature vector obtained by averaging the embeddings of words in its title and abstract. In addition, all papers are also associated with the year that the
the corresponding paper was published.
\subsubsection{The Settings of All Hyperparameters}
\label{sec:hyperparameter}
This section provides more details on the hyperparameter settings of all datasets, which are shown in the following tables.\\\\
\begin{minipage}{\textwidth}
\begin{minipage}[t]{0.48\textwidth}
  \centering
  \scriptsize
    \begin{tabular}{c|c|c|c}
    \hline\hline
    Dataset&Cora&PubMed&Citeseer\\
    \hline
         Learning Rate(GD)&$10^{-1}$&$5\times 10^{-2}$&$10^{-1}$
  \\\hline
         Learning Rate(Adadelta)&$10^{-3}$&$10^{-3}$&$10^{-3}$

 \\\hline
         Learning Rate(Adagrad)&$10^{-3}$&$10^{-3}$&$10^{-3}$\\\hline
         Learning Rate(Adam)&$10^{-4}$&$10^{-4}$&$10^{-3}$\\\hline
         $\rho,\nu$(pdADMM-G)&$10^{-4}$&$10^{-4}$&$10^{-4}$
\\\hline
         $\rho,\nu$(pdADMM-G-Q)&$10^{-4}$&$10^{-3}$&$10^{-3}$

\\\hline\hline
Dataset&\tabincell{c}{Amazon\\ Computers}&\tabincell{c}{Amazon \\Photo}&\tabincell{c}{Coauthor \\CS}\\
\hline
Learning Rate(GD)&$10^{-2}$&$10^{-2}$&$10^{-1}$
\\
\hline
Learning Rate(Adadelta)&$10^{-3}$&$10^{-3}$&$10^{-3}$\\\hline
Learning Rate(Adagrad)&$10^{-3}$&$10^{-3}$&$10^{-3}$\\\hline
Learning Rate(Adam)&$10^{-3}$&$10^{-3}$&$10^{-3}$\\\hline
$\rho,\nu$(pdADMM-G)&$10^{-3}$&$10^{-3}$&$10^{-2}$\\\hline
$\rho,\nu$(pdADMM-G-Q)&$10^{-3}$&$10^{-3}$&$10^{-2}$\\
\hline\hline
 Dataset&\tabincell{c}{Coauthor \\Physics}&Flickr& Ogbn-Arxiv\\
     \hline
         Learning Rate(GD)&$10^{-1}$&$10^{-3}$&$10^{-2}$  \\\hline
         Learning Rate(Adadelta)&$10^{-3}$&$10^{-2}$&$10^{-1}$\\\hline
         Learning Rate(Adagrad)&$10^{-3}$&$10^{-3}$&$10^{-3}$\\\hline
         Learning Rate(Adam)&$10^{-3}$&$10^{-3}$&$10^{-3}$
         \\\hline
         $\rho,\nu$(pdADMM-G)&$10^{-2}$&$10^{-4}$&$10^{-4}$
         \\\hline
         $\rho,\nu$(pdADMM-G-Q)&$10^{-2}$&$10^{-4}$&$10^{-4}$
         \\\hline\hline
    \end{tabular} 
      \makeatletter\def\@captype{table}\makeatother\caption{Hyperparameter settings of all methods on nine benchmark datasets when the number of neurons is 100.}
    \label{tab:hyperparameter_100_neurons}
\end{minipage}
\begin{minipage}[t]{0.48\textwidth}
  \centering\scriptsize
    \begin{tabular}{c|c|c|c}
    \hline\hline
    Dataset&Cora&PubMed&Citeseer\\
    \hline
         Learning Rate(GD)&$10^{-1}$&$5\times10^{-3}$&$10^{-1}$
  \\\hline
         Learning Rate(Adadelta)&$10^{-3}$&$10^{-4}$&$10^{-3}$
 \\\hline
         Learning Rate(Adagrad)&$10^{-3}$&$10^{-3}$&$10^{-3}$\\\hline
         Learning Rate(Adam)&$10^{-4}$&$10^{-4}$&$10^{-4}$\\\hline
         $\rho,\nu$(pdADMM-G)&$10^{-4}$&$10^{-4}$&$10^{-3}$\\\hline
         $\rho,\nu$(pdADMM-G-Q)&$10^{-4}$&$10^{-3}$&$10^{-3}$
\\\hline\hline
Dataset&\tabincell{c}{Amazon\\ Computers}&\tabincell{c}{Amazon \\Photo}&\tabincell{c}{Coauthor \\CS}\\
\hline
Learning Rate(GD)&$10^{-2}$&$10^{-2}$&$10^{-1}$\\
\hline
Learning Rate(Adadelta)&$10^{-3}$&$10^{-3}$&$10^{-3}$\\\hline
Learning Rate(Adagrad)&$10^{-3}$&$10^{-3}$&$10^{-3}$\\\hline
Learning Rate(Adam)&$10^{-4}$&$10^{-4}$&$10^{-4}$\\\hline
$\rho,\nu$(pdADMM-G)&$10^{-3}$&$10^{-3}$&$10^{-3}$\\\hline
$\rho,\nu$(pdADMM-G-Q)&$10^{-3}$&$10^{-3}$&$10^{-3}$\\
\hline\hline
 Dataset&\tabincell{c}{Coauthor \\Physics}&Flickr& Ogbn-Arxiv\\
     \hline
         Learning Rate(GD) &$10^{-2}$&$10^{-2}$&$10^{-2}$\\\hline
         Learning Rate(Adadelta)&$10^{-3}$&$10^{-2}$&$10^{-1}$\\\hline
         Learning Rate(Adagrad)&$10^{-3}$&$10^{-3}$&$10^{-3}$\\\hline
         Learning Rate(Adam)&$10^{-4}$&$10^{-3}$&$10^{-3}$\\\hline
         $\rho,\nu$(pdADMM-G)&$10^{-2}$&$10^{-4}$&$10^{-4}$
         \\\hline
         $\rho,\nu$(pdADMM-G-Q)&$10^{-2}$&$10^{-4}$&$10^{-4}$
         \\\hline\hline
    \end{tabular}
         \makeatletter\def\@captype{table}\makeatother\caption{Hyperparameter settings of all methods on nine benchmark datasets when the number of neurons is 500.}
    \label{tab:hyperparameter_500_neurons}
\end{minipage}

\end{minipage}
\subsubsection{The Performance of Validation Sets}
\label{sec:validation performance}
This section provides more experimental results on the validation sets of all datasets, which are shown in the following tables.\\\\
\begin{minipage}{\textwidth}
\begin{minipage}[t]{0.48\textwidth}
  \centering
  \scriptsize
    \begin{tabular}{c|c|c|c}
    \hline\hline
    Dataset&Cora&PubMed&Citeseer\\
    \hline
         GD&0.704$\pm$0.037&0.626$\pm$	0.072&0.619$\pm$0.045
  \\\hline
         Adadelta&0.652$\pm$0.064&0.720$\pm$0.035&0.620$\pm$0.022

 \\\hline
         Adagrad&0.720$\pm$	0.022&0.762$\pm$	0.012&0.604	$\pm$0.027\\\hline
         Adam&0.720$\pm$0.034&0.745$\pm$	0.014&0.624$\pm$0.014\\\hline
         pdADMM-G&0.750$\pm$0.005&0.788$\pm$0.004&\textbf{0.724$\pm$0.005}
\\\hline
         pdADMM-G-Q&\textbf{0.754$\pm$	0.002}&\textbf{0.793$\pm$0.002}&0.722$\pm$0.002

\\\hline\hline
Dataset&\tabincell{c}{Amazon\\ Computers}&\tabincell{c}{Amazon \\Photo}&\tabincell{c}{Coauthor \\CS}\\
\hline
GD&0.654$\pm$0.033&0.730$\pm$0.165&0.875$\pm$0.007
\\
\hline
Adadelta&0.136$\pm$0.062&0.343$\pm$0.046&0.781$\pm$0.084\\\hline
Adagrad&0.750$\pm$0.095&0.808$\pm$0.018&0.889$\pm$0.006\\\hline
Adam&0.740$\pm$0.010&\textbf{0.850$\pm$0.006}&0.887$\pm$0.009\\\hline
pdADMM-G&\textbf{0.753$\pm$0.005}&0.846$\pm$0.014&0.913$\pm$0.003\\\hline
pdADMM-G-Q&0.688$\pm$0.063&0.822$\pm$0.013&\textbf{0.916$\pm$0.003}\\
\hline\hline
 Dataset&\tabincell{c}{Coauthor \\Physics}&Flickr& Ogbn-Arxiv\\
     \hline
         GD&0.921$\pm$0.009&0.464$\pm$0.008&0.378$\pm$0.004  \\\hline
         Adadelta&0.918$\pm$0.014&0.461$\pm$0.006&0.514$\pm$0.014\\\hline
         Adagrad&0.928$\pm$0.005&0.480$\pm$0.003&0.574$\pm$0.008\\\hline
         Adam&0.919$\pm$	0.010&0.512$\pm$0.004&\textbf{0.681$\pm$0.003}
         \\\hline
         pdADMM-G&0.933$\pm$0.001&\textbf{0.514$\pm$0.001}&0.649$\pm$0.012
         \\\hline
         pdADMM-G-Q&\textbf{0.935$\pm$0.002}&0.506$\pm$0.004&0.661$\pm$0.004
         \\\hline\hline
         
    \end{tabular} 
     \makeatletter\def\@captype{table}\makeatother\caption{The validation performance of all methods when the number of neurons is 100.}
    \label{tab:performance_100_neurons_val}
\end{minipage}
\begin{minipage}[t]{0.48\textwidth}
  \centering\scriptsize
    \begin{tabular}{c|c|c|c}
    \hline\hline
    Dataset&Cora&PubMed&Citeseer\\
    \hline
         GD&0.731$\pm$0.018&0.651$\pm$0.034&0.679$\pm$0.008

  \\\hline
         Adadelta&0.716$\pm$0.061&0.688$\pm$0.024&0.597$\pm$0.025
 \\\hline
         Adagrad&0.765$\pm$0.014&0.776$\pm$0.006&0.668$\pm$0.028\\\hline
         Adam&\textbf{0.758$\pm$0.013}&0.778$\pm$0.008&0.668$\pm$0.020\\\hline
         pdADMM-G&0.753$\pm$0.004&\textbf{0.792$\pm$0.004}&0.729$\pm$0.003
\\\hline
         pdADMM-G-Q&0.757$\pm$0.005&\textbf{0.792$\pm$0.003}&\textbf{0.730$\pm$0.004}
\\\hline\hline
Dataset&\tabincell{c}{Amazon\\ Computers}&\tabincell{c}{Amazon \\Photo}&\tabincell{c}{Coauthor \\CS}\\
\hline
GD&0.727$\pm$	0.012&0.809$\pm$0.012&0.897$\pm$0.003\\
\hline
Adadelta&0.246$\pm$0.073&0.371$\pm$0.075&0.884$\pm$0.003\\\hline
Adagrad&0.766$\pm$0.011&0.860$\pm$0.003&\textbf{0.912$\pm$0.004}\\\hline
Adam&0.750$\pm$0.017&\textbf{0.872$\pm$0.020}&0.893$\pm$0.013\\\hline
pdADMM-G&\textbf{0.778$\pm$0.007}&0.861$\pm$0.005&\textbf{0.912$\pm$0.003}\\\hline
pdADMM-G-Q&0.764$\pm$0.008&0.850$\pm$0.009&0.910$\pm$0.003\\
\hline\hline
 Dataset&\tabincell{c}{Coauthor \\Physics}&Flickr& Ogbn-Arxiv\\
     \hline
         GD &0.928$\pm$0.001&0.466$\pm$0.001&0.451$\pm$0.033\\\hline
         Adadelta&0.932$\pm$0.006&0.462$\pm$0.004&0.591$\pm$0.017\\\hline
         Adagrad&\textbf{0.935$\pm$0.005}&0.488$\pm$0.007&0.646$\pm$0.010\\\hline
         Adam&0.933$\pm$0.007&\textbf{0.516$\pm$0.002}&\textbf{0.692$\pm$0.008
         }\\\hline
         pdADMM-G&0.932$\pm$0.001&0.514$\pm$0.003&0.661$\pm$0.005
         \\\hline
         pdADMM-G-Q&0.933$\pm$0.002&0.514$\pm$0.001&0.667$\pm$0.003

         \\\hline\hline
         
    \end{tabular}
   \makeatletter\def\@captype{table}\makeatother\caption{ The validation performance of all methods when the number of neurons is 500.}
    \label{tab:performance_500_neurons_val}
\end{minipage}

\end{minipage}

\end{document}